%% file: neurips/neurips_2026.tex
\documentclass{article}

\usepackage[main, final, nonatbib]{neurips/neurips_2026}
\usepackage[numbers]{natbib}

\usepackage[utf8]{inputenc} 
\usepackage[T1]{fontenc}    
\usepackage{hyperref}       
\usepackage{url}            
\usepackage{booktabs}       
\usepackage{amsfonts}       
\usepackage{nicefrac}       
\usepackage{microtype}      
\usepackage{xcolor}         

\usepackage{amsmath,amsthm,amssymb,dsfont}
\usepackage{cleveref}
\usepackage{tikz}
\usetikzlibrary{positioning,fit,arrows.meta}
\usepackage{enumitem}
\usepackage{csquotes}
\usepackage{mathtools}

\input{neurips/macro}

\title{Fixed Universal Transformers}

\author{
Jingwen Liu \\
Columbia University \\
New York, NY \\
\texttt{jingwenliu@cs.columbia.edu} \\
\And
Alexandr Andoni \\
Columbia University \\
New York, NY \\
\texttt{andoni@cs.columbia.edu} \\
\And
Daniel Hsu \\
Columbia University \\
New York, NY \\
\texttt{djhsu@cs.columbia.edu}
}

\usepackage{color-edits}
\addauthor{dh}{orange}

\let\citet\citep

\usepackage[export]{adjustbox}

\begin{document}

\maketitle

\begin{abstract}
%
%
%
We introduce \emph{universal transformers}: fixed transformers that can simulate any transformer in a given class via a suitable input embedding.
Analogous to a universal Turing machine, the input embedding encodes a description of the target model while all internal parameters remain fixed. 
We provide explicit sparse constructions achieving universality when the embedding dimension is sufficiently large, and further show that universality is generic: randomly initialized transformers are universal almost surely, which aligns with recent empirical results of Zhong and Andreas (2024).
We empirically validate our theory on the algorithmic tasks of parenthesis balancing and multi-hop reasoning.
Our results suggest that much of a transformer's expressive power may reside in its input representation rather than its learned weights.
\end{abstract}

\section{Introduction}

We introduce the notion of a \emph{universal transformer}: a fixed transformer that can simulate any transformer in a prescribed class by changing only the input embedding.
In a universal transformer, the embedding plays the role of a program: it encodes the parameters of the target transformer, while the attention and value parameters of the universal transformer remain fixed.
This notion can be regarded as a deep learning
analogue of the concept of universality from the theory of computation, of which universal Turing machines~\citep{turing1936computable} and
universal circuits~\citep{valiant1976universal} are the prime examples.
A universal Turing machine is a single Turing machine that is capable of simulating any other Turing machine (provided in a suitable encoding) as input, much like how we regard general purpose computers today~\citep{turing1936computable}.
A universal circuit is a single circuit that can simulate any target circuit of comparable size by simply fixing additional special ``control'' inputs to an appropriate values specific to the target circuit~\citep{valiant1976universal}.
Note that universality here is not the same as universal approximation, which refers to the ability to choose the model parameters to approximate an arbitrary function. 
Instead, the architecture and non-embedding parameters are fixed, and the question is, for any target transformer of comparable form, whether there is  a suitable representation of the input that makes this fixed computation emulate the target transformer.

Our notion of universality is partly motivated by recent experimental findings of \citet{zhong2024algorithmic} on \emph{random transformers}, where they found that many basic capabilities of language models are already present in randomly-initialized transformers, with only trained embedding and unembedding layers. These random transformers turn out to be universal (almost surely). Moreover, randomness is not essential: universal transformers can also be \emph{explicitly} constructed.
Our constructions give more insight into how the emulation can be carried out, and they do not resemble (and are much sparser than) any commonly-used initialization scheme for transformers.
Universal transformers are also related to model reprogramming, a transfer-learning paradigm that adapts a pre-trained model to new domains by changing only the embedding layers~\citep{chen2024model}; the existence of universal transformers vastly expands the potential scope of that approach.

\begin{figure}[t]
  \centering
  \begin{tabular}{ccc}
    \includegraphics[scale=0.41,valign=m]{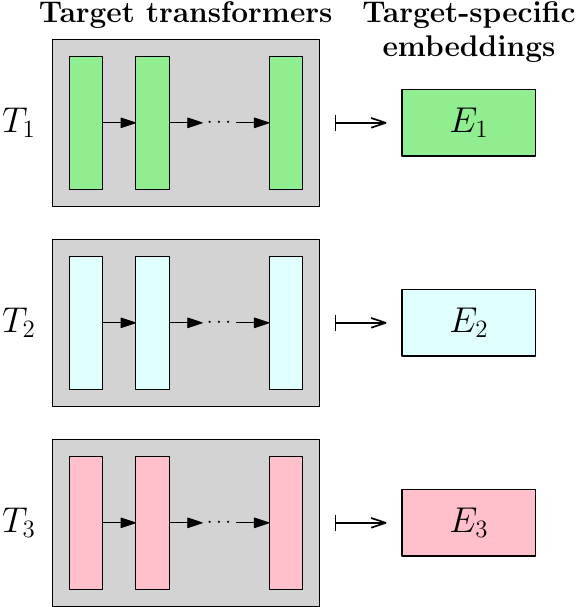}
    & \includegraphics[scale=0.41,valign=m]{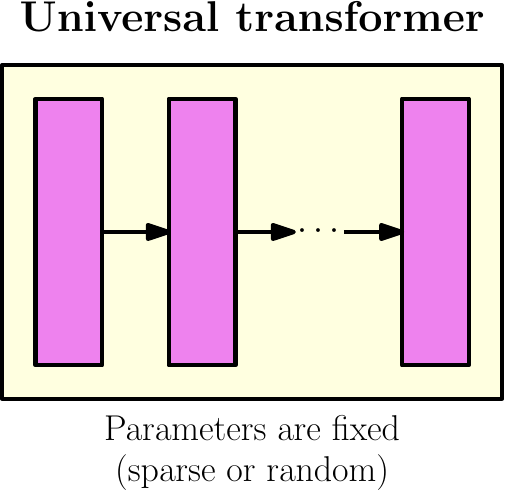}
    & \includegraphics[scale=0.41,valign=m]{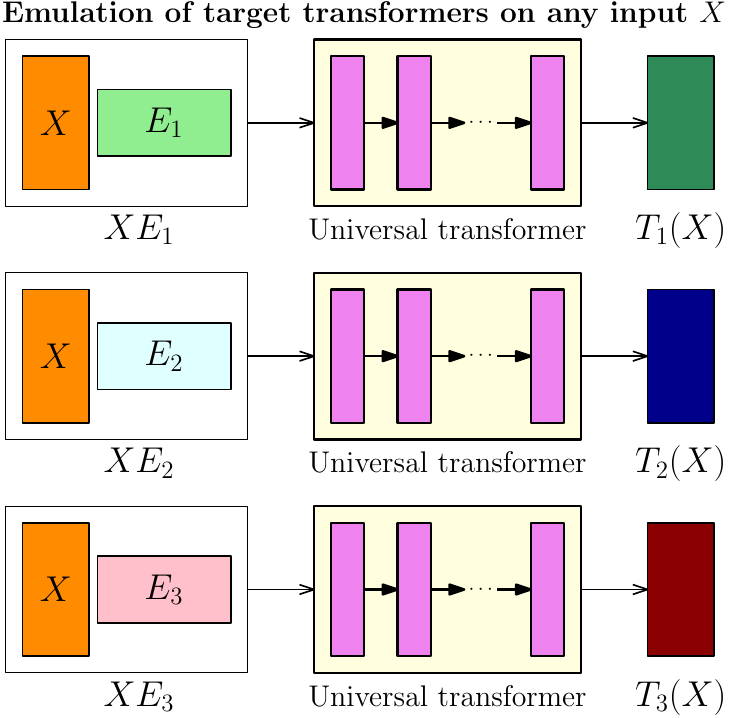}
    \\
    (a) & (b) & (c)
  \end{tabular}
  \caption{%
    Schematic overview of universal transformers.
    (a) Each target transformer corresponds to a target-specific embedding (based on the target's model parameters).
    (b) The parameters of the universal transformer are fixed, either using the sparse deterministic construction from \Cref{thm:sparse_universal_transformer}, or chosen randomly as in \Cref{thm:random_universal_transformer}.
    (c) The composition of the universal transformer with the target-specific embedding emulates the target transformer on all inputs.%
    \label{fig:overview}
  }
\end{figure}

\paragraph{Our contributions.}
We formalize the notion of a universal transformer, which can simulate any transformer within a given class via a suitable input embedding. We provide explicit constructions of such models and further show that randomly initialized transformers are universal.

\begin{itemize}[leftmargin=*]
    \item \textbf{Explicit sparse universal transformers.}
    We show that there exists a fixed $H$-head, $L$-layer universal transformer with embedding dimension $m$ and $\{0,1\}$-valued parameters that can simulate any $H$-head, $L$-layer transformer with embedding dimension $d$, provided that $m=O(Ld)$ when $H=1$ and $m = O(H^L d)$ when $H\ge 2$. Moreover, this result extends to the looped (weight-tied) setting, where the same parameters are reused across layers in the universal transformer.
    \begin{theorem}[Informal version of \Cref{thm:sparse_universal_transformer,thm:looped_sparse_universal_transformer}]
        There exists fixed $\{0,1\}$-valued parameters defining a $H$-head $L$-layer transformer (or looped transformer) $\UT$ with embedding dimension $m=O(Ld)$ when $H=1$ and $m=O(H^Ld)$ when $H\ge 2$,  such that there are at most $m$ nonzero entries in each parameter matrix, and
        for any $H$-head $L$-layer transformer with embedding dimension $d$, there exists an embedding matrix $E\in \R^{d \times m}$, we have $\UT(XE)=T(X)$ for any input X. 
    \end{theorem}
    \item \textbf{Random universal transformers.} We prove that universality is in abundance: almost surely, a randomly initialized $H$-head, $L$-layer transformer can simulate any $H$-head, $L$-layer transformer with embedding dimension $d$, when the embedding dimension satisfies $m = O(H^L d)$.
    \begin{theorem}[Informal version of \Cref{thm:random_universal_transformer,thm:random_looped_universal_transformer}]
        Draw random parameters from an absolutely continuous distribution. With probability $1$, the resulting $H$-head $L$-layer transformer (or looped transformer) $\UT$ with embedding dimension $m=O(H^Ld)$ such that for any $H$-head $L$-layer transformer with embedding dimension $d$, there exists an embedding matrix $E\in \R^{d \times m}$,  we have $\UT(XE)=T(X)$ for any input $X$. 
    \end{theorem}
    \item \textbf{Lower bounds.} We establish that, when the number of heads is $H\ge 2$, the exponential dependence on depth $L$ is unavoidable, even when the embedding dimension is $1$, for any approach based on ``attention-matching'', where the universal transformer reproduces the attention matrices of the transformer being simulated.
    \begin{theorem}[Informal version of \Cref{thm:lower_bound_utf}]
        If an $H$-head, $L$-layer transformer $\UT$ with embedding dimension $m$ can simulate any $H$-head, $L$-layer transformer $T$ with embedding dimension $1$ via ``attention-matching'', then $m\geq H^{L/2}$.
    \end{theorem}
    \item \textbf{Empirical evaluation.} We evaluate sparse and random universal transformers on parenthesis balancing and $k$-hop induction head tasks, achieving perfect or near-perfect accuracy. These findings complement our representational results and demonstrate that standard gradient-based optimization can effectively learn the required embeddings.
\end{itemize}

\subsection{Related work}
\paragraph{Universality in computation.}

%

Our notion of a universal transformer takes inspiration from classical notions of universality in the theory of computation~\citep{turing1936computable,valiant1976universal}, but adapts it to the deep learning context.
Specifically, the description of the target transformer is provided as an input embedding layer for the universal transformer.

As mentioned in the introduction, universal transformers can also be regarded as a form of model reprogramming~\citep{chen2024model}.
Theoretical analyses on model reprogramming rely on distributional assumptions and offer average case error guarantees~\citep{yang2021voice2series}, much like other works on transfer learning.
Our results offer a complementary type of theoretical guarantee for model reprogramming when the ``pre-trained model'' is a universal transformer, and the target tasks are solvable using task-specific transformers.

\paragraph{Universal approximation and Turing-completeness.}

The notion of universality we study in this work is distinct from the universal approximation property of various types of neural networks (including transformers)~\citep[e.g.,][]{cybenko1989approximation,hornik1989multilayer,funahashi1989approximate,yun2020transformers,wei2022statistically}, which is concerned with the ability to arbitrarily-well approximate broad classes of functions by entire families of neural networks.
In particular, for any given target function, the approximating network's size and internal parameters may all be specific to the target function.
Similarly, many network families have been shown to satisfy Turing-completeness, which (in this context) means that any Turing machine can be simulated by an appropriate member of the network family~\citep{siegelmann1992computational,dehghani2019universal,chung2021turing,wei2022statistically,giannou2023looped}.
(We have an unfortunate clash of terminology with \citet{dehghani2019universal}, who use the term ``Universal Transformer'' for a recurrent extension of the standard transformer architecture that the authors argue is Turing-complete; our work does not change the transformer architecture.)
Again, the particular network that simulates a given Turing machine may have size and internal parameters that depend on that Turing machine.

The Turing-completeness of transformers offers one avenue to construct universal transformers: use the fact that every transformer can be simulated by a Turing machine (e.g., one that implements a Python interpreter with PyTorch) when provided in a suitable encoding (e.g., model \texttt{.pt} file), and then obtain a transformer that simulates \emph{that} Turing machine.
However, such an approach would not (at least readily) come with any efficiency or succinctness guarantees about the constructed universal transformer, even for the classes of transformers we consider in our results.


\paragraph{Random networks.}

Random networks have long served as a lens for studying the capabilities of a wide range of machine learning models (especially at random initialization), including kernel machines~\citep{neal1996priors,rahimi2007random}, feedforward neural networks~\citep{jacot2018neural,frankle2019lottery}, recurrent neural networks~\citep{jaeger2001echo,maass2002real,chatziafratis2022scrambling}, and---more recently---transformers~\citep{chen2020lottery,zhong2024algorithmic,dong2025random,otsuka2026strong}.
As mentioned in the introduction, our work is, in part, inspired by the empirical study of \citet{zhong2024algorithmic} on random transformers, which tested and validated the hypothesis that many capabilities of language models are already present in randomly-initialized transformers, prior to any training.
%
%
That work also empirically observed that the token representations at different levels of a random transformers seem to lie close to low-dimensional subspaces corresponding to the current task at hand, but also that these representations could not be sparsified by selecting individual neurons in the multi-layer Perceptron (MLP) layers; this is taken as evidence against a ``lottery ticket hypothesis''~\citep{frankle2019lottery,otsuka2026strong} for the success of random transformers.

The work of \citet{dong2025random} continues the investigation of random transformers by further empirically investigating the relative importance of the self-attention parameters and MLP parameters.
Theoretically, they attribute the expressivity of ``Frozen-QK'' transformers---where only the self-attention parameters are fixed at their initialization, but the MLP parameters are trained---to the universal approximation capabilities of MLPs and an embedding dimension that is sufficiently large for self-attention to gather the entire input sequence into a single token embedding.
This explanation for the capabilities of (a different type of) random transformers is complementary to that provided by universality.
Our theoretical work focuses on attention-only transformers (as opposed to transformers with MLP layers), considers exact representation (as opposed to approximation), and does not use any target-specific layers besides the embedding layers.

As already pointed out by \citep{zhong2024algorithmic}, the lottery ticket hypothesis concerns a rather different property of random networks that involves finding useful subnetworks through sparsification/pruning.
This hypothesis has been explored both empirically~\citep{chen2020lottery} and theoretically~\citep{otsuka2026strong} for transformers, but it is not a basis for the universality property that we study.

\section{Preliminaries}

\paragraph{Notation.}
For a positive integer $k$, we denote $[k] := \{1,2,\dots,k\}$. 
We use $\{0, 1\}^{m \times n}$ to denote the class of $m \times n$ matrices with entries in $\set{0,1}$.
For a matrix $X \in \R^{n \times d}$, we interpret the $n$ rows as tokens, where $n$ is the context length, and $d$ as the feature dimension. 
We use $\softmax$ to denote the row-wise softmax function. 

\subsection{Standard transformer}

We consider a standard multi-head self-attention transformer without MLP layers.

\paragraph{Transformer class.}
We use $\tf_{H, L, \din, d}$ to denote the class of $H$-head $L$-layer transformers, with
the input and output dimension $\din$, and each head uses an internal head dimension $d$.

A transformer $T\in \tf_{H, L, \din, d}$ is defined by parameters including query and key matrices $W_Q^{l,h}, W_K^{l,h} \in \R^{\din \times d}$, value matrix $W_V^{l,h} \in \R^{\din \times d}$, and output matrix $W_O^{l,h} \in \R^{d \times \din}$, for each layer $l \in [L]$ and head $h \in [H]$, and the following computation:

Given an input sequence $X^0 \in \R^{n \times \din}$, $T$ computes for each layer $l \in [L]$:
\begin{align*}
A^{l,h} 
&:= \softmax\!\left( (X^{l-1} W_Q^{l,h})(X^{l-1} W_K^{l,h})^\top \right) 
\in \R^{n \times n}, \\
X^l 
&:= \sum_{h=1}^H A^{l,h} X^{l-1} W_V^{l,h} W_O^{l,h} 
\in \R^{n \times \din}.
\end{align*}
The final output is $T(X) := X^L \in \R^{n \times \din}$.

\subsection{Universal transformer}

We now define the \emph{universal transformer} model. Like universal circuits, we define a universal transformer with respect to a class of transformers of certain set of hyper-parameters.

\paragraph{Universal transformer model.}
We define a parametric model $\utf_{H, L, \din, d}^m$ which is an $H$-head, $L$-layer transformer with embedding dimension $m$. It is parameterized by matrices $R_Q^{l,h}, R_K^{l,h} \in \R^{m \times d}$ and $R_V^{l,h} \in \R^{m \times m}$ for each $(l,h) \in [L] \times [H]$, together with an unembedding matrix $U \in \R^{m \times \din}$.

Given an input sequence $X^0 \in \R^{n \times \din}$ and an embedding matrix $E \in \R^{\din \times m}$, $\utf_{H, L, \din, d}^m$ first maps the input into the $m$-dimensional space:
\[
\tilde X^0 := X^0 E \in \R^{n \times m}.
\]
For each layer $l \in [L]$, it then computes
\begin{align*}
\tilde A^{l,h} 
&:= \softmax\!\left( (\tilde X^{l-1} R_Q^{l,h})(\tilde X^{l-1} R_K^{l,h})^\top \right) 
\in \R^{n \times n}, \\
\tilde X^l 
&:= \sum_{h=1}^H \tilde A^{l,h} \tilde X^{l-1} R_V^{l,h} 
\in \R^{n \times m}.
\end{align*}
Finally, the output is projected back to the input space:
\[
\utf_{H, L, \din, d}^m(X, E) := \tilde X^L U \in \R^{n \times \din}.
\]

The key property of the universal transformer is that it can simulate any transformer in $\tf_{H, L, \din, d}$ via a suitable choice of embedding. 

\begin{definition}[Universal transformer]
    We call a $\utf_{H, L, \din, d}^m$, parametrized by fixed $\{R_Q^{l,h}, R_K^{l,h}, R_V^{l,h}, U\}$,  a \emph{universal transformer} for the class $\tf_{H, L, \din, d}$ if, for any $T \in \tf_{H, L, \din, d}$, there exists an embedding matrix $E$ such that
    \[
    \utf_{H, L, \din, d}^m(X, E) = T(X)
    \quad \text{for all } X \in \R^{n \times \din} \text{ and all } n \in \mathbb{N}.
    \]
\end{definition}

Note that the universal transformer must be \emph{uniform over the context length} $n$: 
the parameters $\{R_Q^{l,h}, R_K^{l,h}, R_V^{l,h}, U\}$ are fixed and do not depend on $n$, 
and the same universal transformer can be applied to inputs of arbitrary length.

This concept is analogous to concepts of universal Turing machines~\citep{turing1936computable} and universal circuits of~\citep{valiant1976universal}, adapted to the deep learning context.
In our setting, the embedding matrix $E$ 
serves as a \emph{description} (or encoding) of the target transformer being simulated, 
while the parameters $\{R_Q^{l,h}, R_K^{l,h}, R_V^{l,h}, U\}$ define a fixed computational 
template that performs the simulation. The embedding dimension $m$ therefore quantifies the expressive capacity of the universal transformer, 
i.e., the amount of information needed to encode a target transformer within this fixed architecture. 
In \Cref{sec:contructing_utf,sec:optimality}, we study the minimal $m$ required for universality over the class $\tf_{H, L, \din, d}$.

We also consider a \emph{looped} (or weight-tied) variant of the universal transformer, in which the same attention parameters are reused across iterations.

\paragraph{Looped universal transformer model.}
We define a looped parametric model $\looped\utf_{H, L, \din, d}^m$ which is an $H$-head transformer unrolled for $L$ iterations, with embedding dimension $m$. It is parameterized by shared matrices $R_Q^{h}, R_K^{h} \in \R^{m \times d}$ and $R_V^{h} \in \R^{m \times m}$ for each $h \in [H]$, together with an unembedding matrix $U \in \R^{m \times \din}$. The $\looped\utf_{H, L, \din, d}^m$ follows the same computation path as $\utf_{H, L, \din, d}^m$ but replacing $R_Q^{l,h}, R_K^{l,h}, R_V^{l,h}$ with $R_Q^{h}, R_K^{h}, R_V^{h}$.


\section{Constructing universal transformers}
\label{sec:contructing_utf}

In this section we characterize when a transformer can serve as a \emph{universal transformer} for the class $\tf_{H, L, \din, d}$, and to determine the minimal embedding dimension $m$ required for universality.

\subsection{Fixed sparse universal transformer}
In the following theorem, we give an explicit construction of a \emph{sparse} universal transformer, which establishes an upper bound on the required embedding dimension and serves as the basis for our experiments. The upper bound is given in terms of the following quantity $\mupper$:
\[
\mupper := \begin{cases}
    (L+1)\max\{2d, \din\}, & \text{if } H = 1, \\
    \frac{2H(H^{L}-1)}{H-1} d + H^L\din , & \text{if } H > 1.
\end{cases}
\]
\begin{theorem}[Sparse universal transformer]
\label{thm:sparse_universal_transformer}
There exist fixed parameters $\{R_Q^{l,h}, R_K^{l,h}, R_V^{l,h}\}_{l \in [L],\, h \in [H]}$ and an unembedding matrix $U$, with at most $m$ nonzero entries in each parameter matrix, defining a $\utf_{H, L, \din, d}^m$ such that it is a universal transformer for the class $\tf_{H, L, \din, d}$, and 
\[
R_Q^{l,h}, R_K^{l,h} \in \{0,1\}^{m \times d}, 
\quad
R_V^{l,h} \in \{0,1\}^{m \times m}, 
\quad
U \in \{0,1\}^{m \times \din},
\]
if the embedding dimension satisfies $m \geq \mupper$.
\end{theorem}

\paragraph{Proof sketch.}
The proof proceeds by showing that the universal transformer can simultaneously match (i) the attention patterns and (ii) the value transformations of any transformer.
The outputs of any target transformer $T$ in the class $\tf_{H, L, \din, d}$ and universal transformer can be written as
\begin{align*}
T(X^0) & = \sum_{(h_1,\dotsc,h_L) \in \sbr{H}^L} A^{L,h_L} \dotsm A^{1,h_1} X^0 W_V^{1,h_1} W_O^{1,h_1} \dotsm W_V^{L,h_L} W_O^{L,h_L} \\
\utf_{H, L, \din, d}^m(X^0, E) & = \sum_{(h_1,\dotsc,h_L) \in \sbr{H}^L} \tilde A^{L,h_L} \dotsm \tilde A^{1,h_1} X^0 E R_V^{1,h_1} \dotsm R_V^{L,h_L} U
\end{align*}
First, we ensure that for every layer $l \in [L]$ and head $h \in [H]$, the attention matrices match:
$
\tilde A^{l,h} = A^{l,h}.
$
This guarantees that the universal transformer follows the same attention pattern over tokens as the target transformer.
Second, conditioned on matching the attention matrices, we construct the value transformations so that 
$ E R_V^{1,h_1} \dotsm R_V^{L,h_L} U= W_V^{1,h_1} W_O^{1,h_1} \dotsm W_V^{L,h_L} W_O^{L,h_L}
$.
Therefore, the overall input-output map is identical.
This reduces the problem to implementing the correct linear transformations along each attention path.

\paragraph{Construction.}
We explicitly construct the parameters to realize this behavior. The key idea is to encode the weights of the target transformer into the embedding matrix $E$, and each parameter matrix is either routing or reading out from certain block of $E$. The sparsity follows from the fact that each matrix only interacts with a small number of coordinates. 

We interpret the embedding space as an $H$-ary tree, indexed by $(h_1, \dots, h_l) \in [H]^l$ for $l\leq L$, each corresponding to a head sequence. These blocks represent all possible computation paths through the transformer.
For each layer $l$ and head $h$, the matrix $R_V^{l,h}$ is a diagonal $0/1$ \emph{mask}. It keeps the nodes of depth $\ge l$ whose $l$-th head index is $h$, and zeros everything else. So $R_V^{1,h_1}\cdots R_V^{k,h_k}$ keeps exactly the nodes of depth $\ge k$ that extend $p=(h_1,\dots,h_k)$.
$R_Q^{l,h}$ (resp.\ $R_K^{l,h}$) is the sum, over all nodes of depth $l-1$, of the selectors of their $h$-th query (resp.\ key) sub-block. $U$ is the sum of all leaf selectors.

We further provide a sparse construction for looped (weight-tied) transformers, where the same parameters are reused across layers. This leads to a more succinct universal architectural description where the number of parameter matrices is $O(H)$ instead of $O(HL)$ in the standard setting. 

\begin{theorem}[Sparse looped universal transformer]
\label{thm:looped_sparse_universal_transformer}
There exist fixed parameters $\{R_Q^{h}, R_K^{h}, R_V^{h}\}_{h \in [H]}$ and an unembedding matrix $U$, with at most $m$ nonzero entries in each parameter matrix, defining a $\looped\utf_{H, L, \din, d}^m$ such that it is a looped universal transformer for the class $\tf_{H, L, \din, d}$, and
\[
R_Q^{h}, R_K^{h} \in \{0,1\}^{m \times d}, 
\quad
R_V^{h} \in \{0,1\}^{m \times m}, 
\quad
U \in \{0,1\}^{m \times \din},
\]
when the embedding dimension satisfies $m \geq \mupper$.
\end{theorem}

Note that all the above constructions work even if there is attention masking, e.g., causal masking for decoder type of transformer, since we are matching all the attention matrices. 

\subsection{Random universal transformer}
The sparse constructions above demonstrate that universality is achievable and establish a sufficient embedding dimension. However, they rely on carefully designed, highly structured parameters. This raises a natural question: \emph{How special are universal transformers?} We now show that universality is in fact a generic property. In particular, a randomly initialized transformer is universal almost surely.

\begin{theorem}[Random universal transformer]
\label{thm:random_universal_transformer}
Let $\{R_Q^{l,h}, R_K^{l,h}, R_V^{l,h}\}_{l \in [L],\, h \in [H]}$ and $U$ be drawn independently from an absolutely continuous distribution over $\R^{m \times d}$, $\R^{m \times d}$, $\R^{m \times m}$, and $\R^{m \times \din}$, respectively (e.g., with i.i.d.\ Gaussian entries). Then, with probability $1$, the resulting transformer $\utf_{H, L, \din, d}^m$ is a universal transformer for the class $\tf_{H, L, \din, d}$, provided that $m \geq \mupper$.
\end{theorem}

The universality condition can be reduced to a collection of algebraic non-degeneracy conditions on the parameters $\{R_Q, R_K, R_V, U\}$. These conditions fail only when certain matrices are rank-deficient or satisfy polynomial constraints. The theorem follows from the fact that such degeneracies form a measure-zero set. This also applies to the looped setting.

\begin{theorem}[Random looped universal transformer]
\label{thm:random_looped_universal_transformer}
Let $\{R_Q^{h}, R_K^{h}, R_V^{h}\}_{h \in [H]}$ and $U$ be drawn independently from an absolutely continuous distribution over $\R^{m \times d}$, $\R^{m \times d}$, $\R^{m \times m}$, and $\R^{m \times \din}$, respectively (e.g., with i.i.d.\ Gaussian entries). Then, with probability $1$, the resulting looped transformer $\looped\utf_{H, L, \din, d}^m$ is a looped universal transformer for the class $\tf_{H, L, \din, d}$, when the embedding dimension satisfies $m \geq \mupper$.
\end{theorem}

These results show that the embedding dimension $m = \Theta(H^L \max\{d, \din\})$ is sufficient for universality, and that almost all transformers of this size are universal.

\paragraph{Finite precision.} In the explicit sparse construction, all universal-transformer parameters are $\{0,1\}$-valued matrices, which are represented exactly in floating point. The embedding $E$ stores the parameters that also appear in the computation path of the target transformer, so the construction introduces no additional ill-conditioned computation; its numerical error is essentially the same error incurred when evaluating the corresponding target transformer at the same precision. For the random construction, the potential source of numerical error arises when solving the system of linear equations. Its magnitude is controlled by the condition number of the random parameters. Thus, the simulation should remain stable for well-conditioned draws, although almost-sure invertibility alone does not provide a uniform finite-precision guarantee.

\subsection{Beyond Attention-Only}
\label{sec:beyond_attn}
The constructions in the previous sections are for attention-only transformers. Practical architectures interleave attention with residual connections, MLP blocks, and normalization. In this section we show that universality extends \emph{exactly} to transformers with residual connections and gated MLPs (GLU variants, including SwiGLU, GeGLU, and ReGLU). We then explain which components break the argument.

\paragraph{MLP blocks.} We can extend both the explicit sparse and random universality constructions to bias-free MLPs with any gated linear unit (GLU) activation $\sigma$ of the form $MLP(X)=[\sigma(XW_g)\odot XW_u]W_d$, where $W_g, W_u\in \R^{\din \times r}, W_d\in \R^{r\times \din}$, $r$ is the MLP width and $\sigma$ is applied elementwise. The proof constructs auxiliary representations backward from the output so that the two linear branches ($XW_g$ and $XW_u$) entering each gate are matched before applying $\sigma$. The main change is that one target neuron (neurons in $W_d$) needs to be simulated by a block of universal neurons. Because the universal transformer’s output directions are fixed, and the target neuron $W_d[j, :] \in \mathbb{R}^{1\times \din}$, we use $\din$ universal neurons which form a set of basis to simulate a target neuron and encode the weights for each basis into the embedding matrix $E$ and propagate through the linear branch ($XW_u$). If the target MLP width is $r$, and we need universal MLP to have width $r\din$. A sufficient embedding dimension for $E$ becomes
$m=O\left((2Hr)^L(d_{\rm in}+d)\right).$ See \Cref{app:glu_extension} for the details.

\paragraph{Residual connections.} Adding residual connections, each layer update becomes $Y^l:=X^{l-1}+\sum_{h=1}^H A^{l,h}X^{l-1}W_V^{l,h} W_O^{l,h}, X^l:=Y^l+\left[\sigma(Y^lW_g^l)\odot(Y^lW_u^l)\right]W_d^l$. The proof extends directly by viewing the residual path as an additional fixed “head” carrying the identity transformation. The same path-unrolling argument then applies after replacing $H$ by $H+1$ in the embedding-dimension bound. See \Cref{app:glu_extension} for details. 

\paragraph{Layer normalization.} Layernorm takes the input and normalize it by the mean and variance across the embedding dimension. For an input $X\in \R^{n\times d}$, define $\mu =\frac{1}{d}X\mathbf{1}_d\in \mathbb{R}^{n \times 1}, X_c = X - \mu \mathbf{1}_d^\top.$ The row-wise variance is $v =\frac{1}{d}(X_c\odot X_c)\mathbf{1}_d\in \mathbb{R}^{n \times 1}.$ Then $\mathrm{LN}(X)=\left(\frac{X-\mu\mathbf{1}_d^\top}{\sqrt{v+\epsilon}\mathbf{1}_d^\top}\right)\odot\mathbf{1}_n\gamma^\top+\mathbf{1}_n\beta^\top$ where the division is elementwise. For universal transformers, embedding into a different dimension $m$ changes the mean and variance across coordinates. Therefore, the linear representation invariant used in our proof is not preserved after normalization. Moreover, it’s impossible for post-LayerNorm with arbitrary learned scale and bias. At $X=0$, we have $XE=0$ for every $E$, so the universal transformer produces one fixed output, whereas a target transformer can produce an arbitrary LayerNorm bias ($\beta$). Similarly, an arbitrary target scale ($\gamma$) can make its output arbitrarily large, while a fixed universal LayerNorm followed by a fixed unembedding has uniformly bounded output. Thus, extending our theorem to standard learned LayerNorm requires either restricting its affine parameters or allowing additional trainable parameters beyond the input embedding.

\section{Optimality}
\label{sec:optimality}

In this section, we establish lower bounds on the embedding dimension $m$ required for universality. Our upper bounds in \Cref{sec:contructing_utf} are achieved by explicitly matching the attention matrices of the target transformer, i.e., enforcing $\tilde A^{l,h} = A^{l,h}$ for all $(l,h)$. 
We prove a lower bound on $m$ for such ``attention-matching'' universal transformers.
To do this, we prove a lower bound on $m$ for a strictly stronger class of models that is additionally \emph{given} the desired attention matrices as input.
This, in turn, implies an $H^{\Omega(L)}$ lower bound on the embedding dimension for attention-matching approaches.

\paragraph{Shared-attention universal computation.}
We formalize this stronger model as follows. The shared-attention universal transformer, denoted $\sutf_{H, L, \din, d}^m$, is parameterized only by value matrices $\{R_V^{l,h}\}_{l \in [L], h \in [H]}$. It takes as input a sequence $X^0 \in \R^{n \times \din}$, an embedding matrix $E \in \R^{\din \times m}$, an unembedding matrix $U \in \R^{m \times \din}$, and attention matrices $\{A^{l,h}\}_{l \in [L], h \in [H]}$. The computation proceeds identically to $\utf_{H, L, \din, d}^m$, except that the attention matrices are externally provided and fixed, i.e., $\tilde A^{l,h}$ is replaced by $A^{l,h}$ for all $(l,h)$.

\begin{theorem}[Lower bound for universal transformers]
\label{thm:lower_bound_utf}
Let $\din = 1$. Consider any transformer class $\tf_{H, L, 1, d}$. If there exists a $\sutf_{H, L, \din, d}^m$ (or $\looped\sutf_{H, L, \din, d}^m$) such that for any $T \in \tf_{H, L, 1, d}$ with attention matrices $\{A^{l,h}\}_{l \in [L], h \in [H]}$,  
\[
\sutf_{H, L, \din, d}^m (X, E, U, \{A^{l,h}\}_{l \in [L], h \in [H]}) = T(X)
\]
for all $X \in \R^{n \times \din}, n \in \mathbb{N}$, and $n \times n$ row-stochastic matrices $A^{l,h}$, then $m \ge H^{L/2}$.
\end{theorem}

Theorem~\ref{thm:lower_bound_utf} shows that the embedding dimension must grow exponentially in the depth $L$, even in the simplest setting $\din = 1$. Notably, this lower bound holds in the shared-attention model, where the desired attention matrices are given as part of the input. Thus, the complexity arises from representing the exponentially-many computation paths induced by the multi-head attention structure.

In particular, across $L$ layers and $H$ heads, there are $H^L$ possible sequences of attention heads, each corresponding to a distinct computation path. The universal transformer must be able to simultaneously represent and combine these paths. The lower bound shows that compressing this exponential-size family of behaviors into a low-dimensional embedding is impossible.

\section{Experimental studies}
In this section, we experimentally evaluate the universal transformers constructed in \Cref{sec:contructing_utf} on two canonical reasoning tasks—parenthesis balancing and $k$-hop induction heads—which are widely used as toy benchmarks for studying algorithmic reasoning in language models \citep{yao-etal-2021-self, sanford2024}.
These experiments complement our theoretical results on representational power by demonstrating empirically that standard gradient-based optimization can successfully learn the required embeddings for universal transformers to solve these tasks.

\subsection{Tasks}

\paragraph{Parenthesis Balancing (Dyck-1).}
In this task, the model is given a sequence of left- and right- parentheses and must output whether the parentheses in the sequence are properly balanced.
Such sequences form the Dyck-1 language, a canonical example of a context-free language~\citep{chompsky1963algebraic}. In our setup, the vocabulary size is $4$: two parenthesis tokens and two labels. The maximum length for the input parentheses sequences is $60$.

\paragraph{$k$-hop induction heads.}
This task generalizes the standard induction head mechanism studied in language models~\citep{elhage2021mathematical,olsson2022context}. In the standard induction heads, the model outputs the token that follows the most recent previous occurrence of the current token in the context. The $k$-hop variant extends this mechanism by recursively composing $k$ such retrieval steps. For example, for an input $X=aa\textcolor{blue}{b} \textcolor{red}{c} \textcolor{red}{c}a \textcolor{blue}{b} \textcolor{red}{c} a$, the $2$-hop induction head output for the last token $a$ is $\textcolor{red}{c}$.
Under a widely-believed complexity assumption, $\Theta(\log k)$ layers are necessary and sufficient to solve this task~\citep{sanford2024}.
Therefore, we use this task to evaluate the performance of deeper universal transformers. For this task, the vocabulary size is $30$, including the number of hops tokens, $4$ character tokens and label tokens. The maximum sequence length is $100$. 

\paragraph{Architectures.}
We implement the universal transformer $\utf_{H, L, \din, d}^m$ using one-hot token representations, so that $\din$ corresponds to the vocabulary size. Following the construction in \Cref{thm:sparse_universal_transformer}, for head dimension $d$, the embedding dimension is given by
$m = \mupper$.
We compare three models: (i) the constructed sparse universal transformer from \Cref{thm:sparse_universal_transformer}, (ii) a random universal transformer with parameters initialized i.i.d.\ from $\text{Uniform}(-1/\sqrt{m},\, 1/\sqrt{m})$ (the default PyTorch initialization), and (iii) a fully trained transformer. Following \citet{zhong2024algorithmic}, for both the sparse and random universal transformers, we train only the embedding matrix $E$ and the unembedding matrix $U$, while keeping all other parameters fixed. For random universal transformers, all the results are averaged across $5$ seeds. 

\subsection{Results}

\Cref{tab:summary_results_combined} summarizes the main results. We first consider the
attention-only setting, which directly matches our theoretical model $\utf_{H, L, \din, d}^m$. On parenthesis balancing, the sparse universal transformer reaches $100.0\%$ accuracy,
matching the fully trained transformer and the random universal also reaches non-trivial accuracy $94.4\%$. This shows that the explicit construction from
\Cref{thm:sparse_universal_transformer} is not merely a worst-case
representation result: the embedding and unembedding matrices can be learned by
gradient descent to solve a nontrivial algorithmic task.

The induction-head tasks are more challenging in the attention-only setting. 
Both sparse and random universal transformers obtain nontrivial accuracies (random guessing gives $20\%$) but far from optimal.
This suggests that while the attention-only universal transformers have the representational capacity, optimization becomes increasingly difficult as the task requires
more depth.
Adding standard components that stabilize training can substantially improve performance. With residual connections and layer normalization, both sparse and random universal
transformers reach much higher accuracy. In particular, the sparse universal transformer reaches $100.0\%$ accuracy on all reported induction-head tasks, while the random universal transformer achieves between $85.7\%$ and $96.1\%$ accuracy.

\begin{table}[t]
\centering
\small
\setlength{\tabcolsep}{3pt}

\begin{subtable}[t]{0.49\linewidth}
\centering
\resizebox{\linewidth}{!}{%
\begin{tabular}{lccc}
\toprule
\textbf{Task} & \textbf{Sparse} & \textbf{Random} & \textbf{Fully-trained} \\
\midrule
Parenthesis Balancing    & 100.0\% & 94.4\% & 100.0\% \\
$1$-hop Induction Head   & 30.2\% & 29.4\% & 8.4\% \\
$2$-hop Induction Head   & 33.2\% & 30.5\% & 51.8\% \\
$3$-hop Induction Head   & 39.2\% & 32.6\% & 99.5\% \\
$4$-hop Induction Head   & 39.4\% & 35.2\% & 95.6\% \\
\bottomrule
\end{tabular}
}
\vspace{0.4em}
\caption{Attention-only.}
\label{tab:summary_results_attn_only}
\end{subtable}
\hfill
\begin{subtable}[t]{0.49\linewidth}
\centering
\resizebox{\linewidth}{!}{%
\begin{tabular}{lccc}
\toprule
\textbf{Task} & \textbf{Sparse} & \textbf{Random} & \textbf{Fully-trained} \\
\midrule
Parenthesis Balancing    & 100.0\% & 100.0\% & 93.0\% \\
$1$-hop Induction Head   & 100.0\% & 96.1\% & 100.0\% \\
$2$-hop Induction Head   & 100.0\% & 85.7\% & 100.0\% \\
$3$-hop Induction Head   & 100.0\% & 92.6\% & 100.0\% \\
$4$-hop Induction Head   & 100.0\% & 89.2\% & 100.0\% \\
\bottomrule
\end{tabular}
}
\vspace{0.4em}
\caption{Adding residual and layer normalization.}
\label{tab:summary_results_layernorm}
\end{subtable}

\caption{Accuracy on parenthesis balancing and $k$-hop induction-head tasks. Parenthesis balancing uses $4$-head, $2$-layer transformers with head dimension $24$ and embedding dimension $1024$. The $1$-, $2$-, $3$-, and $4$-hop induction-head tasks use $2$-head transformers with depths $2,3,4,4$, head dimensions $28,28,30,30$, and embedding dimensions $456,1024,2280,2280$, respectively.}
\label{tab:summary_results_combined}
\end{table}

\begin{table}[t]
\centering
\small
\setlength{\tabcolsep}{3pt}

\begin{subtable}[t]{0.49\linewidth}
\centering
\resizebox{\linewidth}{!}{%
\begin{tabular}{lccc}
\toprule
Model & Sparse & Random & Fully-trained \\
\midrule
Attention-only & 100\% & 94.4\% & 100\% \\
+ residual & 100\% & 95.4\% & 100\% \\
+ residual and LN & 100\% & 100\% & 93.0\% \\
+ MLP & 100\% & 100\% & 100\% \\
+ MLP and LN & 100\% & 100\% & 94.4\% \\
\bottomrule
\end{tabular}
}
\vspace{0.4em}
\caption{Parenthesis Balancing.}
\label{tab:dyck_parenthesis_balancing}
\end{subtable}
\hfill
\begin{subtable}[t]{0.49\linewidth}
\centering
\resizebox{\linewidth}{!}{%
\begin{tabular}{lccc}
\toprule
Model & Sparse & Random & Fully-trained \\
\midrule
Attention-only & 33.2\% & 30.5\% & 51.8\% \\
+ residual & 53.8\% & 54.3\% & 58.4\% \\
+ residual and LN & 100\% & 85.7\% & 100\% \\
+ MLP & 54.6\% & 54.1\% & 100\% \\
+ MLP and LN & 100\% & 66.2\% & 100\% \\
\bottomrule
\end{tabular}
}
\vspace{0.4em}
\caption{2-hop Induction Heads.}
\label{tab:induction_head_2hop}
\end{subtable}
\caption{Ablation study on parenthesis balancing and $2$-hop induction-head tasks. Accuracies are reported for $4$-head, $2$-layer transformers with head dimension $24$ and embedding dimension $1024$ on parenthesis balancing, and for $2$-head, $3$-layer transformers with head dimension $28$ and embedding dimension $1024$ on $2$-hop induction heads.}
\label{tab:ablation_results}
\end{table}

\subsubsection{Ablation studies}

We further study how different components of standard transformers can affect the training optimization performance. 
\Cref{tab:ablation_results} studies the effect of adding residual connections,
layer normalization, and MLP layers. 
The MLP layers are randomly initialized using the PyTorch default initialization scheme. For the sparse and random universal transformers, the MLP layers are also kept fixed and untrained.

On parenthesis balancing, the sparse universal transformer is robust across all variants and achieves $100.0\%$ accuracy in every setting. 
The random universal transformer also improves from $94.4\%$ in the attention-only setting to $100.0\%$ once residual connections and layer normalization or MLP layers are included. 
The $2$-hop induction-head task reveals a sharper distinction between architectural components. Adding layer normalization together with residual connections can largely improve the performance, and this is also robust when adding frozen random MLP layers. 

Overall, the sparse construction gives a useful inductive bias for optimization: it consistently matches or exceeds the random universal transformer when only the embedding and unembedding are trained.

\subsubsection{Embedding dimension scaling}
We next study how performance empirically depends on the embedding dimension $m$. Our theory shows that the universal transformer exists when the embedding dimension grows with the number of attention paths, on the order of $H^L$ up to factors depending on $d$ and $\din$, but this does not account for the optimization difficulty of finding the desired embedding. Therefore, we experimentally investigate how large $m$ must be for standard gradient-based training to find a good solution. \Cref{fig:scaling_law} compares two parameterizations on the $2$-hop induction-head task. In the constructed parameterization, we set $m=\mupper$ according to Theorem~\ref{thm:sparse_universal_transformer}; in the standard parameterization, we instead use $m=Hd$, matching the usual transformer hidden dimension. This gives the empirical evidence that universal transformer can be found when the embedding dimension $m$ is sufficiently large.

\begin{figure}
    \centering
    \begin{subfigure}[t]{0.4\linewidth}
        \centering
        \includegraphics[width=\linewidth]{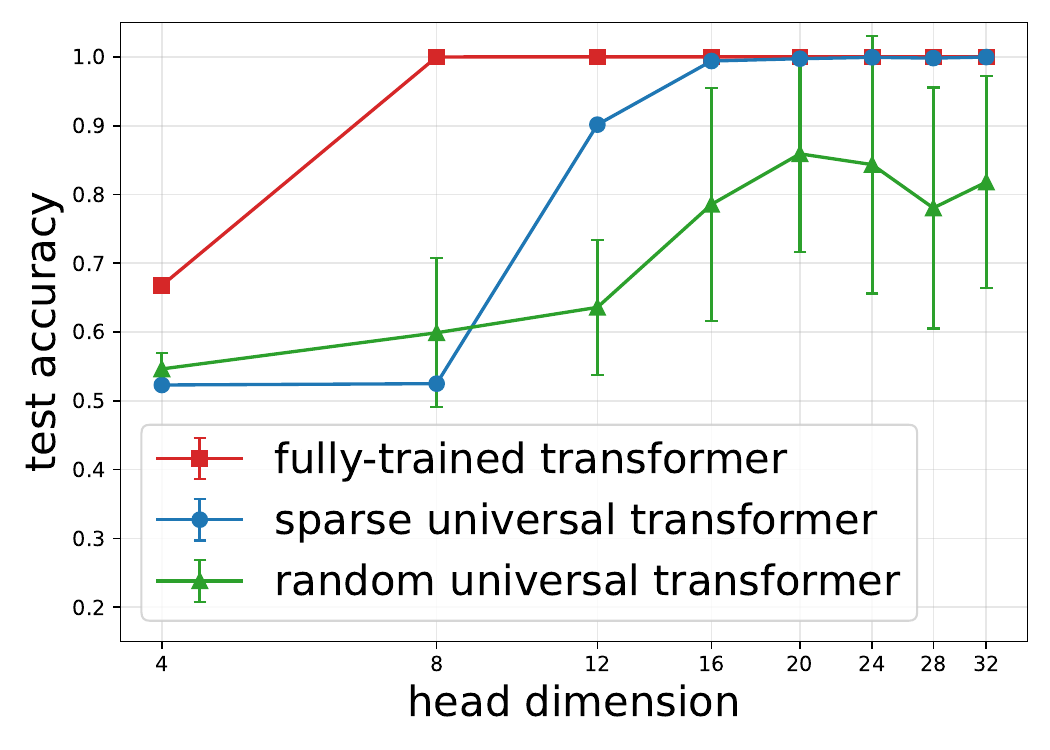}
        \caption{Constructed parameterization $m=\mupper$}
        \label{fig:scaling_law_2_hop}
    \end{subfigure}
    \quad
    \begin{subfigure}[t]{0.4\linewidth}
        \centering
        \includegraphics[width=\linewidth]{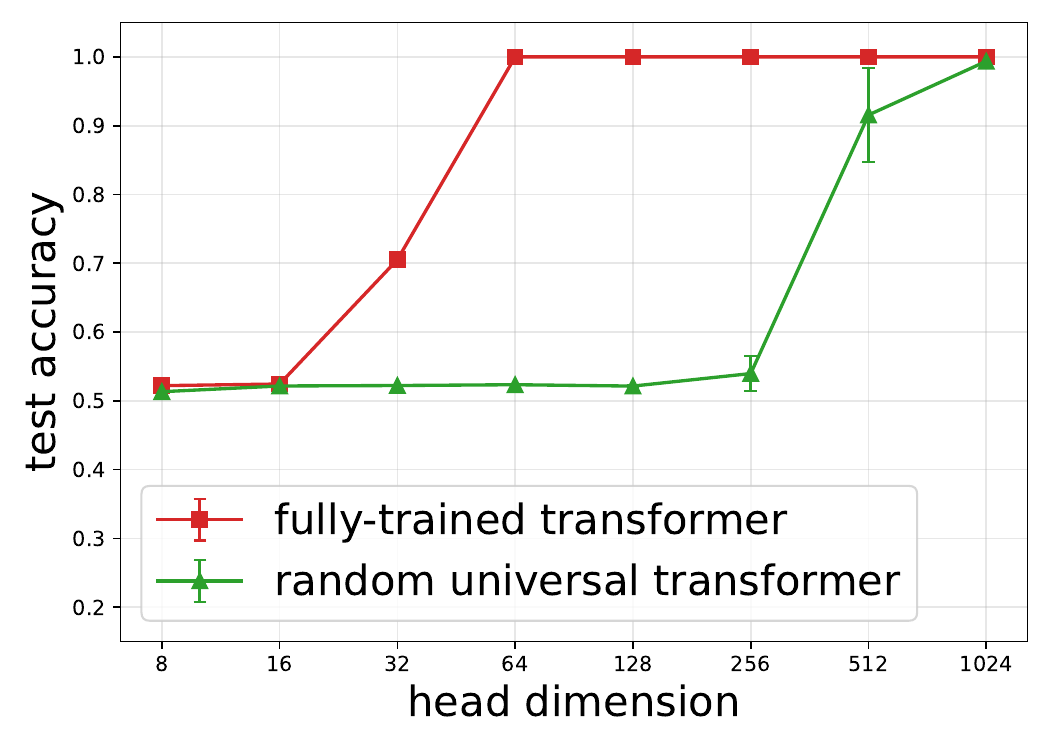}
        \caption{Standard parametrization $m=Hd$}
        \label{fig:scaling_law_2_hop_max_accuracy}
    \end{subfigure}
    \caption{\textbf{Embedding dimension scaling for $2$-hop induction heads.}
    Results are shown for $2$-head, $3$-layer transformers.
    \textit{Left:} Performance under constructed
    parameterization from \Cref{thm:sparse_universal_transformer}, where $m = \mupper$.
    \textit{Right:} Performance under standard parameterization, where $m = Hd$.
    Error bars indicate $\pm 1$ standard deviation across $5$ random seeds. 
    }
    \label{fig:scaling_law}
\end{figure}

\section{Conclusion and future work}
In this paper, we study the universal transformer: a fixed-parameter transformer that can simulate any transformer in a prescribed class by changing only the input embedding. In this view, the embedding matrix acts as a description of the target
model, while the internal attention and value parameters define a fixed computational template. We gave explicit sparse constructions of universal transformers, extended the construction to the looped weight-tied setting, and
showed that universality is ubiquitous: randomly initialized transformers are universal almost surely. 

As discussed in \Cref{sec:beyond_attn}, extending these results to architectures with normalization or MLPs with non-gated activation is an important direction for future work. Our empirical results suggest that this extension is especially promising, since layernorm substantially improve
performance in our experiments. We also note that our lower bound is proved only for a restricted matching-attention model, in which the universal transformer must match the target transformer's attention patterns. Establishing lower bounds for more general forms of simulation remains an important open problem.

\begin{ack}
  This work was supported in part by the National Science Foundation under grant DMS-2502259, the Office of Naval Research under grant N00014-24-1-2700, a JP Morgan Faculty Award, and a Research Award from the Columbia Center of AI Technology in collaboration with Amazon.
  We are grateful to Ezra Edelman, Surbhi Goel and Enric Boix-Adsera for their helpful feedback on the construction of universal transformers.
  We are also grateful to Sanjoy Dasgupta for telling us about \citep{zhong2024algorithmic}.
\end{ack}

\bibliographystyle{plainnat}
\bibliography{reference} 
\newpage
\appendix

\input{neurips/appendix}



\end{document}

%% file: neurips/macro.tex
\Crefformat{equation}{#2(#1)#3}
\crefformat{equation}{#2(#1)#3}
\Crefrangeformat{equation}{#3(#1)#4--#5(#2)#6}
\crefrangeformat{equation}{#3(#1)#4--#5(#2)#6}
\Crefmultiformat{equation}{#2(#1)#3}{ and~#2(#1)#3}{, #2(#1)#3}{ and~#2(#1)#3}
\crefmultiformat{equation}{#2(#1)#3}{ and~#2(#1)#3}{, #2(#1)#3}{ and~#2(#1)#3}

\setlist[description]{font=\normalfont}
\mathtoolsset{centercolon}

\DeclarePairedDelimiterXPP\ind[1]{\mathds{1}}{\lbrace}{\rbrace}{}{#1}
\DeclarePairedDelimiterX\eval[1]{\lbrace}{\rvert}{#1 \delimsize\rbrace}
\DeclarePairedDelimiter\ip{\langle}{\rangle}

\DeclarePairedDelimiter\del{\lparen}{\rparen}
\DeclarePairedDelimiter\sbr{\lbrack}{\rbrack}

\DeclarePairedDelimiter\set{\lbrace}{\rbrace}

\DeclarePairedDelimiterX\Set[1]\lbrace\rbrace{%
  
  #1
}
\DeclarePairedDelimiterX\Sbr[1]\lbrack\rbrack{%
  
  #1
}
\DeclarePairedDelimiterX\Del[1]\lparen\rparen{%
  
  #1
}

\newtheorem{theorem}{Theorem}[section]
\newtheorem*{theorem*}{Theorem}

\newtheorem{lemma}[theorem]{Lemma}
\newtheorem{definition}[theorem]{Definition}
\newtheorem{assumption}[theorem]{Assumption}

\DeclareMathOperator\softmax{softmax}
\DeclareMathOperator\tr{tr}
\DeclareMathOperator\mat{mat}

\newcommand\T{{\scriptscriptstyle{\mathsf{T}}}}
\newcommand\R{\mathbb{R}}

\newcommand\din{d_{\operatorname{in}}}

\makeatletter
\def\ddefloop#1{\ifx\ddefloop#1\else\ddef{#1}\expandafter\ddefloop\fi}

\def\ddef#1{\expandafter\def\csname bf#1\endcsname{\ensuremath{\mathbf{#1}}}}
\ddefloop abcdefghijklmnopqrstuvwxyzABCDEFGHIJKLMNOPQRSTUVWXYZ\ddefloop

\def\ddef#1{\expandafter\def\csname bf#1\endcsname{\ensuremath{\boldsymbol{\csname #1\endcsname}}}}
\ddefloop {alpha}{beta}{gamma}{delta}{epsilon}{varepsilon}{zeta}{eta}{theta}{vartheta}{iota}{kappa}{lambda}{mu}{nu}{xi}{pi}{varpi}{rho}{varrho}{sigma}{varsigma}{tau}{upsilon}{phi}{varphi}{chi}{psi}{omega}{Gamma}{Delta}{Theta}{Lambda}{Xi}{Pi}{Sigma}{varSigma}{Upsilon}{Phi}{Psi}{Omega}{ell}\ddefloop

\def\ddef#1{\expandafter\def\csname cal#1\endcsname{\ensuremath{\mathcal{#1}}}}
\ddefloop ABCDEFGHIJKLMNOPQRSTUVWXYZ\ddefloop
\makeatother

\usepackage{xspace}
\usepackage{subcaption}
\usepackage{booktabs}
\usepackage[normalem]{ulem}
\newcommand{\tf}{\textup{TF}\xspace}
\newcommand{\utf}{\textup{UTF}\xspace}

\newcommand{\looped}{\textup{looped-}\xspace}
\newcommand{\sutf}{\textup{shared-attn-UTF}\xspace}
\newcommand{\UT}{\ensuremath{\operatorname{UT}}}
\newcommand{\mupper}{\overline{m}}

%% file: neurips/appendix.tex
\section{Proof of \Cref{thm:sparse_universal_transformer} (universal transformer)}
\label{app:universal_transformer}
Given an input $X^0$, let $Y:=X^L$ denote the output of the target transformer, and $Z:=\tilde X^L U$ denote the output of the universal transformer.
We prove that for any target transformer, there exists a construction of $E$ such that our constructed universal transformer outputs $Z=Y$ for any input $X^0$.

\subsection{Warmup: single head per layer}

In this section, we assume $H=1$ and drop the $h$ superscript in the notations.

\begin{lemma}[Unrolling]
  \label{lem:unroll}
  Assume $H=1$.
  For each $l \in \sbr{L}$,
  \begin{align*}
    X^l & = A^l \dotsm A^1 X^0 W_V^1 W_O^1 \dotsm W_V^l W_O^l \\
    \tilde X^l & = \tilde A^l \dotsm \tilde A^1 X^0 E R_V^1 \dotsm R_V^l .
  \end{align*}
\end{lemma}

\begin{lemma}[Equal attention]
  \label{lem:eqatt}
  Assume $H=1$, and also that
  \begin{equation}
    \left\{
      \begin{aligned}
        W_V^1 W_O^1 \dotsm W_V^{l-1} W_O^{l-1} W_Q^l & = E R_V^1 \dotsm R_V^{l-1} R_Q^l \\
        W_V^1 W_O^1 \dotsm W_V^{l-1} W_O^{l-1} W_K^l & = E R_V^1 \dotsm R_V^{l-1} R_K^l
      \end{aligned}
    \right.
    \quad \text{for each $l\in\sbr{L}$} .
    \label{eq:QK}
  \end{equation}
  Then for any $X^0 \in \R^{n \times \din}$, we have $A^l = \tilde A^l$ for each $l \in \sbr{L}$.
\end{lemma}
\begin{proof}
  The proof is by induction.
  The base case is $L=1$.
  We have
  \begin{align*}
    X^0 W_Q^1 & = X^0 ER_Q^1 = \tilde X^0 R_Q^1 \\
    X^0 W_K^1 & = X^0 ER_K^1 = \tilde X^0 R_K^1
  \end{align*}
  by \Cref{eq:QK} for $l=1$ and the definition of $\tilde X^0$.
  Using these identities with the definitions of $A^1$ and $\tilde A^1$, we have
  \begin{equation*}
    A^1 = \softmax((X^0 W_Q^1)(X^0 W_K^1)^\T)
    = \softmax((\tilde X^0 R_Q^1)(\tilde X^0 R_K^1)^\T) = \tilde A^1 ,
  \end{equation*}
  This proves the claim for $L=1$.

  For the inductive step, consider $L>1$.
  We assume, as the inductive hypothesis, that for each $l \in \sbr{L-1}$, $A^l = \tilde A^l$.
  We need to show that $A^L = \tilde A^L$.
  We have
  \begin{align*}
    X^{L-1} W_Q^L
    & = A^{L-1} \dotsm A^1 X^0 W_V^1 W_O^1 \dotsm W_V^{L-1} W_O^{L-1} W_Q^L
    && \text{(by \Cref{lem:unroll})} \\
    & = \tilde A^{L-1} \dotsm \tilde A^1 X^0 W_V^1 W_O^1 \dotsm W_V^{L-1} W_O^{L-1} W_Q^L
    && \text{(by inductive hypothesis)} \\
    & = \tilde A^{L-1} \dotsm \tilde A^1 X^0 E R_V^1 \dotsm R_V^{L-1} R_Q^L
    && \text{(by \Cref{eq:QK} for $l=L$)} \\
    & = \tilde X^{L-1} R_Q^L
    && \text{(by \Cref{lem:unroll})} ,
    \\
    \intertext{and similarly,}
    X^{L-1} W_K^L & = \tilde X^{L-1} R_K^L .
  \end{align*}
  Therefore
  \begin{equation*}
    A^L = \softmax((X^{L-1} W_Q^L)(X^{L-1} W_K^L)^\T)
    = \softmax((\tilde X^{L-1} R_Q^L)(\tilde X^{L-1} R_K^L)^\T) = \tilde A^L .
  \end{equation*}
  This completes the inductive step.
\end{proof}

\begin{lemma}[Equal input-output behavior]
  Assume $H=1$, and also that both \Cref{eq:QK} and
  \begin{equation}
    \label{eq:OV}
    W_V^1 W_O^1 \dotsm W_V^L W_O^L
    = E R_V^1 \dotsm R_V^L U
  \end{equation}
  hold.
  Then for any $X^0 \in \R^{n \times \din}$, we have $Y=Z$.
\end{lemma}
\begin{proof}
  We have
  \begin{align*}
    Y = X^L
    & = A^L \dotsm A^1 X^0 W_V^1 W_O^1 \dotsm W_V^L W_O^L
    && \text{(by \Cref{lem:unroll})} \\
    & = \tilde A^L \dotsm \tilde A^1 X^0 W_V^1 W_O^1 \dotsm W_V^L W_O^L
    && \text{(by \Cref{lem:eqatt})} \\
    & = \tilde A^L \dotsm \tilde A^1 X^0 E R_V^1 \dotsm R_V^L U
    && \text{(by \Cref{eq:OV})} \\
    & = \tilde X^L U = Z
    && \text{(by \Cref{lem:unroll})} .
    \qedhere
  \end{align*}
\end{proof}

\begin{lemma}[Construction for single head]
\label{lem:single_head}
Assume $H=1$ and $m \geq (L+1)\max\{2d, \din\}$. Then there exists fixed $\{R_Q^l, R_K^l, R_V^l\}_{l\in [L]}$ and $U$ such that for any transformer $T\in \tf_{1, L, \din, d}$, there exists an $E$ so that \Cref{eq:QK} and \Cref{eq:OV} hold.
\end{lemma}

\begin{proof}
    Let $b := \max\{2d,\din\}$.
    Partition the first $(L+1)b$ embedding coordinates into consecutive blocks
    $B_1,\ldots,B_{L+1}$ of width $b$; any remaining coordinates are unused.
    Let $S_i\in\R^{m\times b}$ be the coordinate selector for $B_i$, so that
    $ES_i$ is the submatrix of $E$ stored in that block.
    Within a block, let $J_Q,J_K\in\R^{b\times d}$ select its first $d$
    and next $d$ coordinates, respectively, and let
    $J_U\in\R^{b\times\din}$ select its first $\din$ coordinates.

    Define a common shift matrix
    \[
        R:=\sum_{i=1}^{L}S_{i+1}S_i^\T\in\R^{m\times m}.
    \]
    Thus $RS_i=S_{i+1}$ for $i\le L$, $RS_{L+1}=0$, and $R$ is zero
    on unused coordinates. Equivalently, right multiplication by $R$ shifts
    the contents of $B_{i+1}$ into $B_i$.
    For every layer $l\in[L]$, set
    \[
        R_V^l:=R,\qquad R_Q^l:=S_1J_Q,\qquad
        R_K^l:=S_1J_K,\qquad U:=S_1J_U.
    \]
    In particular, these matrices are independent of the target transformer
    and are tied across layers. Since $R^tS_1=S_{t+1}$ for $0\le t\le L$,
    their products read the desired blocks of the original embedding.

    Now fix an arbitrary target transformer $T\in\tf_{1,L,\din,d}$ and write
    \[
        W_{VO}^{1:l}:=W_V^1W_O^1\dotsm W_V^lW_O^l,
        \qquad W_{VO}^{1:0}:=I_{\din}.
    \]
    For each $l\in[L]$, fill the first $d$ columns of block $B_l$ with
    $W_{VO}^{1:l-1}W_Q^l$, and the next $d$ columns with
    $W_{VO}^{1:l-1}W_K^l$. Fill the first $\din$ columns of $B_{L+1}$
    with $W_{VO}^{1:L}$. Set all remaining entries of $E$ to zero.
    These assignments do not overlap because $b\ge 2d$ and $b\ge\din$.

    For every $l\in[L]$, the selector identities give
    \begin{align*}
        ER_V^1\dotsm R_V^{l-1}R_Q^l
        &=ER^{l-1}S_1J_Q=ES_lJ_Q=W_{VO}^{1:l-1}W_Q^l,\\
        ER_V^1\dotsm R_V^{l-1}R_K^l
        &=ER^{l-1}S_1J_K=ES_lJ_K=W_{VO}^{1:l-1}W_K^l.
    \end{align*}
    Hence \Cref{eq:QK} holds. Similarly,
    \[
        ER_V^1\dotsm R_V^LU
        =ER^LS_1J_U
        =ES_{L+1}J_U
        =W_{VO}^{1:L},
    \]
    proving \Cref{eq:OV}. The construction works for every target transformer,
    including targets whose weights differ across layers. Each fixed matrix has
    entries in $\{0,1\}$ and at most $m$ nonzero entries.
\end{proof}

\subsection{Multiple heads per layer}

\begin{lemma}[Multi-headed Unrolling]
  \label{lem:multiunroll}
  For each $l \in \sbr{L}$,
  \begin{align*}
    X^l & = \sum_{(h_1,\dotsc,h_l) \in \sbr{H}^l} A^{l,h_l} \dotsm A^{1,h_1} X^0 W_V^{1,h_1} W_O^{1,h_1} \dotsm W_V^{l,h_l} W_O^{l,h_l} \\
    \tilde X^l & = \sum_{(h_1,\dotsc,h_l) \in \sbr{H}^l} \tilde A^{l,h_l} \dotsm \tilde A^{1,h_1} X^0 E R_V^{1,h_1} \dotsm R_V^{l,h_l} .
  \end{align*}
\end{lemma}
\begin{proof}
  We prove the first identity; the second follows by the same argument with $X^0$ replaced by $X^0E$, $(W_V^{t,h_t}W_O^{t,h_t})$ replaced by $R_V^{t,h_t}$, and $(A^{t,h_t})$ replaced by $(\tilde A^{t,h_t})$.
  For $l=1$, the update rule gives
  \[
    X^1=\sum_{h_1=1}^H A^{1,h_1}X^0W_V^{1,h_1}W_O^{1,h_1},
  \]
  which is exactly the stated formula.
  Assume the identity holds for $l-1\ge 1$. Then
  \begin{align*}
    X^l
    &= \sum_{h_l=1}^H A^{l,h_l} X^{l-1} W_V^{l,h_l}W_O^{l,h_l} \\
    &= \sum_{h_l=1}^H A^{l,h_l}
      \left(\sum_{(h_1,\dotsc,h_{l-1})\in\sbr{H}^{l-1}} A^{l-1,h_{l-1}}\dotsm A^{1,h_1} X^0
        \prod_{t=1}^{l-1} W_V^{t,h_t}W_O^{t,h_t}\right)
      W_V^{l,h_l}W_O^{l,h_l} \\
    &= \sum_{(h_1,\dotsc,h_l)\in\sbr{H}^l}
      A^{l,h_l}\dotsm A^{1,h_1} X^0
      \prod_{t=1}^{l} W_V^{t,h_t}W_O^{t,h_t},
  \end{align*}
  which is the desired expansion.
\end{proof}

\begin{lemma}[Equal attention for multi-head]
    \label{lem:multieqatt}
    For each $l \in [L]$, 
    \begin{equation}
        \left\{
          \begin{aligned}
            W_V^{1, h_1} W_O^{1, h_1} \dotsm W_V^{l-1, h_{l-1}} W_O^{l-1, h_{l-1}} W_Q^{l, h_l} & = E R_V^{1, h_1} \dotsm R_V^{l-1, h_{l-1}} R_Q^{l, h_l} \\
            W_V^{1, h_1} W_O^{1, h_1} \dotsm W_V^{l-1, h_{l-1}} W_O^{l-1, h_{l-1}} W_K^{l, h_l} & = E R_V^{1, h_1} \dotsm R_V^{l-1, h_{l-1}} R_K^{l, h_l}
          \end{aligned}
        \right.
        \quad \text{for each $(h_1, \ldots, h_l)\in\sbr{H}^l$} .
    \label{eq:multi-head-QK}
    \end{equation}
  Then for any $X^0 \in \R^{n \times \din}$, we have $A^{l, h_l} = \tilde A^{l, h_l}$ for each $h_l\in [H], l \in \sbr{L}$.
\end{lemma}
\begin{proof}
  Fix a layer $l\in\sbr{L}$ and a head $h_l\in\sbr{H}$.
  It suffices to show that
  \[
    X^{l-1}W_Q^{l,h_l}=\tilde X^{l-1}R_Q^{l,h_l}
    \quad\text{and}\quad
    X^{l-1}W_K^{l,h_l}=\tilde X^{l-1}R_K^{l,h_l},
  \]
  since then by the definitions of attention,
  \[
    A^{l,h_l}
    =\softmax\!\bigl((X^{l-1}W_Q^{l,h_l})(X^{l-1}W_K^{l,h_l})^\T\bigr)
    =\softmax\!\bigl((\tilde X^{l-1}R_Q^{l,h_l})(\tilde X^{l-1}R_K^{l,h_l})^\T\bigr)
    =\tilde A^{l,h_l}.
  \]

  By \Cref{lem:multiunroll}, we can write
  \begin{align*}
    X^{l-1}W_Q^{l,h_l}
    &= \sum_{(h_1,\dotsc,h_{l-1})\in\sbr{H}^{l-1}}
      A^{l-1,h_{l-1}}\dotsm A^{1,h_1}\,
      X^0\Bigl(W_V^{1,h_1}W_O^{1,h_1}\dotsm W_V^{l-1,h_{l-1}}W_O^{l-1,h_{l-1}}W_Q^{l,h_l}\Bigr).
  \end{align*}
  Using the assumption \Cref{eq:multi-head-QK}, each parenthesized product equals
  \[
    E\,R_V^{1,h_1}\dotsm R_V^{l-1,h_{l-1}}R_Q^{l,h_l}.
  \]
  Substituting and factoring out the rightmost $R_Q^{l,h_l}$ gives
  \begin{align*}
    X^{l-1}W_Q^{l,h_l}
    &= \sum_{(h_1,\dotsc,h_{l-1})\in\sbr{H}^{l-1}}
      A^{l-1,h_{l-1}}\dotsm A^{1,h_1}\,
      X^0E\,R_V^{1,h_1}\dotsm R_V^{l-1,h_{l-1}}R_Q^{l,h_l}.
  \end{align*}
  If we assume (inductively in $t$) that for all $t\le l-1$ and all heads $h_t$, we already have $A^{t,h_t}=\tilde A^{t,h_t}$, then the above becomes
  \begin{align*}
    X^{l-1}W_Q^{l,h_l}
    &= \sum_{(h_1,\dotsc,h_{l-1})\in\sbr{H}^{l-1}}
      \tilde A^{l-1,h_{l-1}}\dotsm \tilde A^{1,h_1}\,
      X^0E\,R_V^{1,h_1}\dotsm R_V^{l-1,h_{l-1}}R_Q^{l,h_l} \\
    &= \tilde X^{l-1}R_Q^{l,h_l},
  \end{align*}
  by the second identity in \Cref{lem:multiunroll}. The $W_K/R_K$ case is identical.

  Finally, we prove the claim layer-by-layer.
  For $l=1$, the computation of $A^{1,h_1}$ only depends on $X^0$ and the pair $(W_Q^{1,h_1},W_K^{1,h_1})$, and the assumption \Cref{eq:multi-head-QK} (with $l=1$) directly gives
  $X^0W_Q^{1,h_1}=\tilde X^0R_Q^{1,h_1}$ and $X^0W_K^{1,h_1}=\tilde X^0R_K^{1,h_1}$, hence $A^{1,h_1}=\tilde A^{1,h_1}$.
  Now assume for some $l\ge 2$ that $A^{t,h}=\tilde A^{t,h}$ holds for all $t\le l-1$ and all heads $h\in\sbr{H}$.
  Repeating the argument above with this inductive hypothesis yields
  $X^{l-1}W_Q^{l,h_l}=\tilde X^{l-1}R_Q^{l,h_l}$ and $X^{l-1}W_K^{l,h_l}=\tilde X^{l-1}R_K^{l,h_l}$ for each $h_l$,
  and therefore $A^{l,h_l}=\tilde A^{l,h_l}$ for each $h_l$.
  This completes the induction.
\end{proof}

\begin{lemma}[Equal input-output behavior for multi-head]
  Assume \Cref{eq:multi-head-QK} and
  \begin{equation}
    \label{eq:multi-head-OV}
    W_V^{1, h_1} W_O^{1, h_1} \dotsm W_V^{L, h_L} W_O^{L, h_L}
    = E R_V^{1, h_1} \dotsm R_V^{L, h_L} U \quad \text{for each $(h_1, \ldots, h_L)\in [H]^L$}
  \end{equation}
  hold.
  Then for any $X^0 \in \R^{n \times \din}$, we have $Y=Z$.
\end{lemma}
\begin{proof}
  Fix $X^0\in\R^{n\times\din}$. By \Cref{lem:multieqatt}, we have
  \[
    A^{l,h}=\tilde A^{l,h}\qquad\text{for all }l\in\sbr{L},\; h\in\sbr{H}.
  \]
  Using the multi-headed unrolling \Cref{lem:multiunroll} at $l=L$, we can write
  \begin{align*}
    Y = X^L
    &= \sum_{(h_1,\dotsc,h_L)\in\sbr{H}^L}
      A^{L,h_L}\dotsm A^{1,h_1}\,X^0
      \Bigl(W_V^{1,h_1}W_O^{1,h_1}\dotsm W_V^{L,h_L}W_O^{L,h_L}\Bigr) \\
    &= \sum_{(h_1,\dotsc,h_L)\in\sbr{H}^L}
      \tilde A^{L,h_L}\dotsm \tilde A^{1,h_1}\,X^0
      \Bigl(W_V^{1,h_1}W_O^{1,h_1}\dotsm W_V^{L,h_L}W_O^{L,h_L}\Bigr).
  \end{align*}
  By the assumption \Cref{eq:multi-head-OV}, each product in parentheses equals
  \[
    E\,R_V^{1,h_1}\dotsm R_V^{L,h_L}U.
  \]
  Substituting this in, and then applying \Cref{lem:multiunroll} to recognize the resulting sum as $\tilde X^L$, yields
  \begin{align*}
    Y
    &= \sum_{(h_1,\dotsc,h_L)\in\sbr{H}^L}
      \tilde A^{L,h_L}\dotsm \tilde A^{1,h_1}\,X^0E\,
      R_V^{1,h_1}\dotsm R_V^{L,h_L}U \\
    &= \tilde X^L U \\
    &= Z,
  \end{align*}
  completing the proof.
\end{proof}

\begin{theorem}[Fixed sparse $R_Q,R_K,R_V$]
\label{thm:designed-R}
Assume $H>1$ and $m\ge N$, where
\[
N := \sum_{t=0}^{L-1} H^t\cdot (2Hd) + H^L\din
=O(H^Ld+H^L\din).
\]
There exist fixed sparse matrices $R_Q^{l,h},R_K^{l,h}\in\R^{m\times d}$,
$R_V^{l,h}\in\R^{m\times m}$ for $l\in[L]$, $h\in[H]$, and
$U\in\R^{m\times\din}$ such that for every target transformer
$T\in\tf_{H,L,\din,d}$ there exists $E\in\R^{\din\times m}$ satisfying
\Cref{eq:multi-head-QK} and \Cref{eq:multi-head-OV}.
\end{theorem}

\begin{proof}
\textbf{Construction idea.}
We organize the active columns of $E$ as an $H$-ary tree
(\Cref{fig:E-block}). The root stores the first-layer query and key matrices
for all heads. Each internal node indexed by a head prefix stores the query
and key matrices for the next layer, premultiplied by the value-output
product along that prefix. Each leaf stores the complete value-output
product along its path. The value maps act as diagonal masks: after choosing
head $h$ at layer $l$, they retain all descendant blocks whose $l$-th head
is $h$ and set the other blocks to zero. Thus they preserve the stored
coefficients needed at subsequent layers. The query and key selectors read
the appropriate subblocks at the current depth, and $U$ reads the leaves.

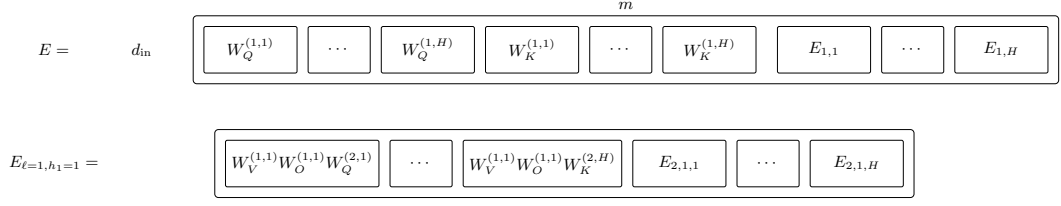
\begin{figure}[ht]
    \centering
    \resizebox{\linewidth}{!}{%
    \begin{tikzpicture}[
      font=\small,
      box/.style={draw, rounded corners=1pt, minimum height=9mm, minimum width=12mm, align=center},
      sbox/.style={draw, rounded corners=1pt, minimum height=9mm, minimum width=18mm, align=center},
      bigfit/.style={draw, rounded corners=2pt, inner sep=2mm},
      arr/.style={-Latex, thick}
    ]
      \node at (-2.6,0) {$E=$};
      \node at (-0.9,0) {$\din$};

      \node[sbox] (wq11) at (1.2,0) {$W_Q^{(1,1)}$};
      \node[box, right=2mm of wq11] (dotsA) {$\cdots$};
      \node[sbox, right=2mm of dotsA] (wq1H) {$W_Q^{(1,H)}$};
      \node[sbox, right=2mm of wq1H] (wk11) {$W_K^{(1,1)}$};
      \node[box, right=2mm of wk11] (dotsB) {$\cdots$};
      \node[sbox, right=2mm of dotsB] (wk1H) {$W_K^{(1,H)}$};
      \node[sbox, right=4mm of wk1H] (e11) {$E_{1,1}$};
      \node[box, right=2mm of e11] (dotsC) {$\cdots$};
      \node[sbox, right=2mm of dotsC] (e1H) {$E_{1,H}$};

      \node[bigfit, fit=(wq11)(e1H), label=above:$N$] (etop) {};

      \node at (-2.6,-2.2) {$E_{\ell=1,h_1=1}=$};

      \node[sbox] (wvowq) at (2.2,-2.2) {$W_V^{(1,1)}W_O^{(1,1)}W_Q^{(2,1)}$};
      \node[box, right=2mm of wvowq] (dotsD) {$\cdots$};
      \node[sbox, right=2mm of dotsD] (wvowk) {$W_V^{(1,1)}W_O^{(1,1)}W_K^{(2,H)}$};
      \node[sbox, right=2mm of wvowk] (e211) {$E_{2,1,1}$};
      \node[box, right=2mm of e211] (dotsE) {$\cdots$};
      \node[sbox, right=2mm of dotsE] (e21H) {$E_{2,1,H}$};
      \node[bigfit, fit=(wvowq)(e21H)] (ebot) {};

    \end{tikzpicture}%
    }%
    \caption{Active block structure of $E$ and a schematic recursive expansion. Any additional embedding coordinates are set to zero.}
    \label{fig:E-block}
    \end{figure}

For convenience, define
\[
M^{l,h}:=W_V^{l,h}W_O^{l,h}\in\R^{\din\times\din},
\qquad
M^{1:t,p}:=\prod_{s=1}^t M^{s,h_s}
\quad\text{for }p=(h_1,\ldots,h_t),
\]
where products are ordered with increasing $s$ from left to right and
$M^{1:0,\emptyset}:=I_{\din}$.

\textbf{Coordinate blocks.}
Partition the first $N$ coordinates into disjoint blocks indexed by prefixes.
For every $t\in\{0,\ldots,L-1\}$ and $p\in[H]^t$, allocate an internal
block $B_{t,p}$ of width $2Hd$. For every $p\in[H]^L$, allocate a leaf
block $C_p$ of width $\din$. Partition each internal block into subblocks
\[
B_{t,p}
=
\Bigl(\bigsqcup_{h=1}^H B^Q_{t,p;h}\Bigr)
\sqcup
\Bigl(\bigsqcup_{h=1}^H B^K_{t,p;h}\Bigr),
\qquad
|B^Q_{t,p;h}|=|B^K_{t,p;h}|=d.
\]
Any coordinates beyond $N$ are unused. For a coordinate block $D$, write
$S_D\in\R^{m\times |D|}$ for its coordinate selector, so that
$ES_D=E_D$.

\textbf{Define the embedding.}
For every internal block, set
\[
E_{B^Q_{t,p;h}}=M^{1:t,p}W_Q^{t+1,h},
\qquad
E_{B^K_{t,p;h}}=M^{1:t,p}W_K^{t+1,h}.
\]
For every leaf, set
\[
E_{C_p}=M^{1:L,p}.
\]
Set all unused columns of $E$ to zero. These assignments occupy disjoint
blocks and specify an embedding for every target transformer.

\textbf{Define the value maps.}
For each $l\in[L]$ and $h\in[H]$, let $R_V^{l,h}$ be diagonal, with
\[
(R_V^{l,h})_{ii}
=
\begin{cases}
1,& i\in B_{t,p},\ t\ge l,\ p_l=h,\\
1,& i\in C_p,\ p_l=h,\\
0,& \text{otherwise}.
\end{cases}
\]
Here $p_l$ denotes the $l$-th entry of the prefix $p$. Thus each map retains
exactly the relevant descendant blocks, in their original coordinates.
For a prefix $p=(h_1,\ldots,h_k)$, write
\[
D_p:=R_V^{1,h_1}\cdots R_V^{k,h_k},
\qquad D_{\emptyset}:=I_m.
\]
Since these matrices are diagonal, $D_p$ retains an internal block
$B_{t,q}$ precisely when $t\ge k$ and the first $k$ entries of $q$ equal
$p$; it retains a leaf $C_q$ precisely when its first $k$ entries equal
$p$. This also holds on all active blocks for $k=0$.

\textbf{Define the query, key, and output readouts.}
For $l\in[L]$ and $h\in[H]$, set
\[
R_Q^{l,h}:=\sum_{p\in[H]^{l-1}}S_{B^Q_{l-1,p;h}},
\qquad
R_K^{l,h}:=\sum_{p\in[H]^{l-1}}S_{B^K_{l-1,p;h}},
\qquad
U:=\sum_{p\in[H]^L}S_{C_p}.
\]
The sums have dimensions $m\times d$, $m\times d$, and
$m\times\din$, respectively. These matrices are zero on unused coordinates.

\textbf{Verify the multi-head $QK$ equalities.}
Fix $l\in[L]$, $h_l\in[H]$, and $p=(h_1,\ldots,h_{l-1})\in[H]^{l-1}$.
Among the blocks at depth $l-1$, the mask $D_p$ preserves precisely the
block indexed by $p$. Consequently,
\[
D_pR_Q^{l,h_l}=S_{B^Q_{l-1,p;h_l}},
\qquad
D_pR_K^{l,h_l}=S_{B^K_{l-1,p;h_l}}.
\]
This includes $l=1$, when $D_p=I_m$ and the root is the only depth-zero
block. Multiplication by $E$ therefore gives
\begin{align*}
ER_V^{1,h_1}\cdots R_V^{l-1,h_{l-1}}R_Q^{l,h_l}
&=E_{B^Q_{l-1,p;h_l}}=M^{1:l-1,p}W_Q^{l,h_l},\\
ER_V^{1,h_1}\cdots R_V^{l-1,h_{l-1}}R_K^{l,h_l}
&=E_{B^K_{l-1,p;h_l}}=M^{1:l-1,p}W_K^{l,h_l}.
\end{align*}
Thus \Cref{eq:multi-head-QK} holds, and \Cref{lem:multieqatt} implies
$A^{l,h}=\tilde A^{l,h}$ for every layer and head.

\textbf{Verify the multi-head $OV$ equalities.}
For a full path $p=(h_1,\ldots,h_L)$, the mask $D_p$ preserves precisely
its leaf block and annihilates every internal block and every other leaf.
Hence $D_pU=S_{C_p}$, and
\[
ER_V^{1,h_1}\cdots R_V^{L,h_L}U
=E_{C_p}
=M^{1:L,p}
=W_V^{1,h_1}W_O^{1,h_1}\cdots W_V^{L,h_L}W_O^{L,h_L}.
\]
This proves \Cref{eq:multi-head-OV}. By the equal input-output behavior lemma
for multiple heads, the two transformers therefore have the same output
for every $X^0$. All constructed matrices other than $E$ depend only on
$H,L,d,\din,m$, and each has at most one nonzero entry per row.
The same construction works for every $m\ge N$ by the zero padding already
specified, completing the proof.
\end{proof}

\section{Proof of \Cref{thm:random_universal_transformer} (random universal transformer)}
\label{app:random_universal_transformer}

\begin{theorem}[Random parameters]
\label{thm:random-R}
Let $H,L\geq 1$ and assume
\[
m\geq N:=2Hd\sum_{t=0}^{L-1}H^t+\din H^L.
\]
Draw $\{R_Q^{l,h},R_K^{l,h},R_V^{l,h}\}_{l\in[L],h\in[H]}$ and $U$ independently from absolutely continuous distributions on their respective matrix spaces. Then, with probability $1$, for every target transformer $T\in\tf_{H,L,\din,d}$, there exists $E\in\R^{\din\times m}$ such that \Cref{eq:multi-head-QK,eq:multi-head-OV} hold.
\end{theorem}

\begin{proof}
For a word $p=(h_1,\ldots,h_t)\in[H]^t$, define
\[
\mathcal R_p:=R_V^{1,h_1}\cdots R_V^{t,h_t},
\qquad
M_p:=W_V^{1,h_1}W_O^{1,h_1}\cdots W_V^{t,h_t}W_O^{t,h_t},
\]
with $\mathcal R_{\phi}:=I_m$ and $M_{\phi}:=I_{\din}$ for the empty word $\phi$.
There are $H\sum_{t=0}^{L-1}H^t$ matrix constraints for each of $Q$ and $K$, each contributing $d$ columns, and $H^L$ output constraints, each contributing $\din$ columns. Thus the total number of columns is $N$.

\textbf{Stack all constraints into one linear system.}
Fix an ordering of the indices and form the horizontal concatenations
\begin{align*}
R_{\mathrm{all}}
&:=\Bigl[
\{\mathcal R_p R_Q^{t+1,h}\}_{\substack{0\leq t<L,\ p\in[H]^t\\h\in[H]}}
\quad
\{\mathcal R_p R_K^{t+1,h}\}_{\substack{0\leq t<L,\ p\in[H]^t\\h\in[H]}}
\quad
\{\mathcal R_p U\}_{p\in[H]^L}
\Bigr]\in\R^{m\times N},\\
W_{\mathrm{all}}
&:=\Bigl[
\{M_p W_Q^{t+1,h}\}_{\substack{0\leq t<L,\ p\in[H]^t\\h\in[H]}}
\quad
\{M_p W_K^{t+1,h}\}_{\substack{0\leq t<L,\ p\in[H]^t\\h\in[H]}}
\quad
\{M_p\}_{p\in[H]^L}
\Bigr]\in\R^{\din\times N}.
\end{align*}
Then \Cref{eq:multi-head-QK,eq:multi-head-OV} are equivalent to
\begin{equation}
E R_{\mathrm{all}}=W_{\mathrm{all}}.
\label{eq:multi-head-stacked}
\end{equation}
If $R_{\mathrm{all}}$ has full column rank, its left inverse
$R_{\mathrm{all}}^{\dagger}=(R_{\mathrm{all}}^\top R_{\mathrm{all}})^{-1}R_{\mathrm{all}}^\top$
satisfies $R_{\mathrm{all}}^{\dagger}R_{\mathrm{all}}=I_N$, and
$E=W_{\mathrm{all}}R_{\mathrm{all}}^{\dagger}$ solves this system.

\textbf{A nonvanishing determinant polynomial.}
Let $\omega$ collect all entries of the sampled query, key, value, and unembedding matrices.
Every entry of $R_{\mathrm{all}}$ is a polynomial in $\omega$.
Let $\bar R_{\mathrm{all}}$ be its submatrix consisting of the first $N$ rows, and define
\[
p(\omega):=\det(\bar R_{\mathrm{all}}).
\]
To show that this particular polynomial is nonzero, use the deterministic tied construction in the proof of \Cref{thm:designed-R-looped} below. That construction is independent of the present random-parameter argument: it gives, at dimension $N$, matrices for which the stacked columns are precisely the $N$ distinct standard basis vectors. Assign these same tied matrices to every layer of the untied model. If $m>N$, embed the construction in the first $N$ coordinates and set every other matrix entry to zero. At this assignment, $\bar R_{\mathrm{all}}$ is a permutation matrix, so $p(\omega)=\pm1$.

Consequently, $p$ is not the zero polynomial. Its zero set has Lebesgue measure zero, and the joint law of $\omega$ is absolutely continuous. Therefore $R_{\mathrm{all}}$ has full column rank with probability $1$. This event depends only on the sampled parameters, not on the target transformer. On this single event, the formula
$E=W_{\mathrm{all}}R_{\mathrm{all}}^{\dagger}$ works for every target, proving the theorem.

For $H>1$, $N=\frac{2H(H^L-1)}{H-1}d+H^L\din=\mupper$.
For $H=1$, $N=2Ld+\din\leq(L+1)\max\{2d,\din\}=\mupper$.
Thus the embedding dimension assumed in \Cref{thm:random_universal_transformer} suffices in both cases.
\end{proof}

\section{Proof of \Cref{thm:looped_sparse_universal_transformer} (looped universal transformer)}
\label{app:looped_designed}

We give a tied construction whose query/key coordinates and output coordinates are disjoint. Only the simulating transformer is weight-tied; the target transformer may have different parameters at every layer.

\begin{theorem}[Designed $R_Q,R_K,R_V$ for looped transformer]
\label{thm:designed-R-looped}
Let $H,L\geq1$ and assume
\[
m\geq N:=2Hd\sum_{t=0}^{L-1}H^t+\din H^L
=\begin{cases}
2Ld+\din,&H=1,\\
\displaystyle\frac{2H(H^L-1)}{H-1}d+H^L\din,&H>1.
\end{cases}
\]
There exist fixed matrices
\[
R_Q^h,R_K^h\in\{0,1\}^{m\times d},
\qquad R_V^h\in\{0,1\}^{m\times m}
\quad(h\in[H]),
\qquad U\in\{0,1\}^{m\times\din},
\]
each with at most $m$ nonzero entries, such that for every target transformer
$T\in\tf_{H,L,\din,d}$ there exists $E\in\R^{\din\times m}$ satisfying
\Cref{eq:multi-head-QK,eq:multi-head-OV} with $R_Q^{l,h}=R_Q^h$,
$R_K^{l,h}=R_K^h$, and $R_V^{l,h}=R_V^h$.
Consequently, the resulting looped transformer simulates $T$ exactly for every input.
\end{theorem}

\begin{proof}
It suffices to construct the matrices for $m=N$. Larger dimensions are handled by placing this construction in the first $N$ coordinates, setting all remaining entries of the fixed matrices to zero, and appending zero columns to $E$.

\textbf{Separate query/key and output coordinates.}
Index the standard basis of $\R^N$ by the following two disjoint families:
\[
e^{\sigma}_{p,h,j},
\qquad
\sigma\in\{Q,K\},\quad
p\in\bigcup_{t=0}^{L-1}[H]^t,\quad h\in[H],\quad j\in[d],
\]
and
\[
f_{s,a},
\qquad s\in[H]^L,\quad a\in[\din].
\]
The first family has $2Hd\sum_{t=0}^{L-1}H^t$ vectors, and the second has $\din H^L$ vectors, giving exactly $N$ coordinates. The empty word is denoted by $\phi$.

\textbf{Define the tied matrices.}
For each $g\in[H]$, define $R_V^g$ on the query/key basis vectors by
\[
R_V^g e^{\sigma}_{p,h,j}
:=\begin{cases}
e^{\sigma}_{(g,p),h,j},&|p|<L-1,\\
0,&|p|=L-1.
\end{cases}
\]
Thus the operator prepends a head index until the maximum query/key depth is reached. On the output basis vectors, define
\[
R_V^g f_{(s_1,\ldots,s_L),a}
:=f_{(g,s_1,\ldots,s_{L-1}),a}.
\]
This operator prepends $g$ and discards the last head index. For $L=1$, the expression on the right is $f_{(g),a}$.
Define the columns of the query, key, and output matrices by
\[
(R_Q^h)_{:,j}:=e^Q_{\phi,h,j},
\qquad
(R_K^h)_{:,j}:=e^K_{\phi,h,j},
\qquad
U_{:,a}:=f_{(1,\ldots,1),a}.
\]
These matrices are fixed independently of the target. Every column of each $R_V^g$ is either zero or a standard basis vector, so $R_V^g$ has at most $N$ nonzero entries. Each $R_Q^h$ and $R_K^h$ has $d$ nonzero entries, and $U$ has $\din$ nonzero entries. All entries are in $\{0,1\}$.

For $p=(h_1,\ldots,h_t)$, write
\[
R_p:=R_V^{h_1}\cdots R_V^{h_t},
\qquad R_{\phi}:=I_N.
\]
Because matrix products act on column vectors from right to left, prepending indices gives
\[
(R_pR_Q^h)_{:,j}=e^Q_{p,h,j},
\qquad
(R_pR_K^h)_{:,j}=e^K_{p,h,j}
\qquad (0\leq |p|\leq L-1).
\]
After exactly $L$ output shifts, all entries of the initial word $(1,\ldots,1)$ have been discarded, and therefore
\[
(R_sU)_{:,a}=f_{s,a}
\qquad (s\in[H]^L).
\]
In particular, the stacked columns
$\{R_pR_Q^h\}$, $\{R_pR_K^h\}$, and $\{R_sU\}$ are exactly the $N$ distinct standard basis vectors. This fact will also supply the full-rank witness for the random-parameter results.

\textbf{Encode the arbitrary target in $E$.}
For $p=(h_1,\ldots,h_t)$, define the layer-indexed target product
\[
M_p:=W_V^{1,h_1}W_O^{1,h_1}\cdots W_V^{t,h_t}W_O^{t,h_t},
\qquad M_{\phi}:=I_{\din}.
\]
Let $\varepsilon_j\in\R^d$ and $\eta_a\in\R^{\din}$ denote standard basis vectors. Specify $E$ on its coordinate basis by
\begin{align*}
E e^Q_{p,h,j}&:=M_pW_Q^{t+1,h}\varepsilon_j,
&&|p|=t<L,\\
E e^K_{p,h,j}&:=M_pW_K^{t+1,h}\varepsilon_j,
&&|p|=t<L,\\
E f_{s,a}&:=M_s\eta_a,
&&|s|=L.
\end{align*}
Since all these coordinates are disjoint, the prescriptions define one matrix $E\in\R^{\din\times N}$ without conflicting assignments.

\textbf{Verify the matching conditions.}
The basis identities above imply
\[
ER_pR_Q^h=M_pW_Q^{t+1,h},
\qquad
ER_pR_K^h=M_pW_K^{t+1,h}
\qquad (|p|=t<L),
\]
and
\[
ER_sU=M_s\qquad (|s|=L).
\]
These are precisely \Cref{eq:multi-head-QK,eq:multi-head-OV} for the tied simulating parameters and the layer-dependent target parameters. The equal-attention and equal-input-output lemmas therefore give
$\tilde A^{l,h}=A^{l,h}$ for every layer and head, and
$\tilde X^LU=X^L$ for every input.

The construction also applies when $H=1$: the query/key coordinates form two families of chains, and each output coordinate is fixed by $R_V^1$. Its dimension is $2Ld+\din$. For $H\geq2$, the displayed dimension is $O(H^L(d+\din))$.
\end{proof}

\section{Proof of \Cref{thm:random_looped_universal_transformer} (random looped universal transformer)}
\label{app:random_looped}

\begin{theorem}[Random looped parameters]
\label{thm:looped_random-R}
Let $H,L\geq1$ and assume
\[
m\geq N:=2Hd\sum_{t=0}^{L-1}H^t+\din H^L.
\]
Let $\{R_Q^h,R_K^h,R_V^h\}_{h\in[H]}$ and $U$ be drawn independently from absolutely continuous distributions on their respective matrix spaces, where
$R_Q^h,R_K^h\in\R^{m\times d}$, $R_V^h\in\R^{m\times m}$, and $U\in\R^{m\times\din}$.
Then, with probability $1$, for every target transformer $T\in\tf_{H,L,\din,d}$ there exists $E\in\R^{\din\times m}$ such that
\begin{align}
M_pW_Q^{t+1,h}&=ER_pR_Q^h,
&&p\in[H]^t,\ 0\leq t<L,\ h\in[H],
\label{eq:looped-random-Q}\\
M_pW_K^{t+1,h}&=ER_pR_K^h,
&&p\in[H]^t,\ 0\leq t<L,\ h\in[H],
\label{eq:looped-random-K}\\
M_s&=ER_sU,
&&s\in[H]^L.
\label{eq:looped-random-OV}
\end{align}
Here, for $p=(h_1,\ldots,h_t)$,
\[
R_p:=R_V^{h_1}\cdots R_V^{h_t},
\qquad
M_p:=W_V^{1,h_1}W_O^{1,h_1}\cdots W_V^{t,h_t}W_O^{t,h_t},
\]
with $R_{\phi}=I_m$ and $M_{\phi}=I_{\din}$. Consequently, the sampled looped transformer simulates every such target through a suitable embedding $E$.
\end{theorem}

\begin{proof}
\textbf{Stack all constraints into one linear system.}
Using the same ordering for both matrices, define
\begin{align*}
R_{\mathrm{all}}
&:=\Bigl[
\{R_pR_Q^h\}_{\substack{0\leq t<L,\ p\in[H]^t\\h\in[H]}}
\quad
\{R_pR_K^h\}_{\substack{0\leq t<L,\ p\in[H]^t\\h\in[H]}}
\quad
\{R_sU\}_{s\in[H]^L}
\Bigr]\in\R^{m\times N},\\
W_{\mathrm{all}}
&:=\Bigl[
\{M_pW_Q^{t+1,h}\}_{\substack{0\leq t<L,\ p\in[H]^t\\h\in[H]}}
\quad
\{M_pW_K^{t+1,h}\}_{\substack{0\leq t<L,\ p\in[H]^t\\h\in[H]}}
\quad
\{M_s\}_{s\in[H]^L}
\Bigr]\in\R^{\din\times N}.
\end{align*}
The conditions in the theorem are equivalent to
\begin{equation}
ER_{\mathrm{all}}=W_{\mathrm{all}}.
\label{eq:looped-random-stacked}
\end{equation}
There are $H\sum_{t=0}^{L-1}H^t$ blocks of width $d$ in each query/key group and $H^L$ blocks of width $\din$ in the output group, giving the stated column count $N$.

\textbf{A nonzero polynomial witness.}
Let $x$ consist of all entries of the tied matrices $\{R_Q^h,R_K^h,R_V^h\}_{h\in[H]}$ and the sampled matrix $U$.
Each entry of $R_{\mathrm{all}}$ is a polynomial in $x$, including when a tied matrix occurs repeatedly in a product. Define
\[
p(x):=\sum_{\substack{I\subseteq[m]\\|I|=N}}
\det\bigl(R_{\mathrm{all}}[I,:]\bigr)^2.
\]
Then $p(x)=0$ if and only if $R_{\mathrm{all}}$ has rank smaller than $N$.
The construction in \Cref{thm:designed-R-looped}, embedded in the first $N$ coordinates if necessary, makes these stacked columns distinct standard basis vectors. Hence one $N\times N$ minor has determinant $\pm1$, and $p$ is not the zero polynomial.

The zero set of $p$ has Lebesgue measure zero. Because the joint distribution of all sampled parameters, including $U$, is absolutely continuous, we have
\[
\Pr\bigl(\operatorname{rank}(R_{\mathrm{all}})=N\bigr)=1.
\]
On this event, keep the sampled $U$ fixed and set
\[
E:=W_{\mathrm{all}}R_{\mathrm{all}}^{\dagger},
\qquad
R_{\mathrm{all}}^{\dagger}
:=(R_{\mathrm{all}}^\top R_{\mathrm{all}})^{-1}R_{\mathrm{all}}^\top.
\]
Since $R_{\mathrm{all}}^{\dagger}R_{\mathrm{all}}=I_N$, this solves \eqref{eq:looped-random-stacked}. The full-rank event is independent of the target, so the same sampled backbone and unembedding work for every target, with only $E$ changing. The matching conditions then imply identical attention matrices and input-output behavior, as required.
\end{proof}

\section{Extension to SwiGLU feed-forward blocks}
\label{app:glu_extension}

We extend the sparse and random constructions to bias-free attention layers followed by SwiGLU feed-forward blocks. The proof uses linear probes of each hidden representation and places target-dependent coefficients in the linear branch of the gate. Throughout this section, $r$ is the target feed-forward hidden width and $\sigma$ is a fixed entrywise activation. SwiGLU corresponds to $\sigma(t)=t/(1+\exp(-t))$; the argument does not require any additional property of $\sigma$.

\paragraph{Architecture.}
Write $M^{l,h}:=W_V^{l,h}W_O^{l,h}\in\R^{\din\times\din}$. Starting from $X^0\in\R^{n\times\din}$, the target layer computes
\begin{align}
  A^{l,h}
  &:=\softmax\!\left((X^{l-1}W_Q^{l,h})(X^{l-1}W_K^{l,h})^\T\right),\nonumber\\
  S^l&:=\sum_{h=1}^H A^{l,h}X^{l-1}M^{l,h},\nonumber\\
  X^l&:=\left[\sigma(S^lW_g^l)\odot(S^lW_u^l)\right]W_d^l,
  \label{eq:glu_target}
\end{align}
where $\odot$ denotes entrywise multiplication, $W_g^l,W_u^l\in\R^{\din\times r}$ and $W_d^l\in\R^{r\times\din}$. The query and key matrices have the same dimensions as in the attention-only model. For a fixed activation $\sigma$, denote this target class by $\tf^{\sigma}_{H,L,\din,d,r}$.

The universal model has embedding dimension $m$ and layer-dependent feed-forward hidden widths $q_l$. Its fixed parameters are
\begin{align*}
  R_Q^{l,h},R_K^{l,h}&\in\R^{m\times d},
  &R_V^{l,h}&\in\R^{m\times m},\\
  R_g^l,R_u^l&\in\R^{m\times q_l},
  &R_d^l&\in\R^{q_l\times m},
  \qquad U\in\R^{m\times\din}.
\end{align*}
Given an embedding $E\in\R^{\din\times m}$, it starts from $\tilde X^0=X^0E$ and computes
\begin{align}
  \tilde A^{l,h}
  &:=\softmax\!\left((\tilde X^{l-1}R_Q^{l,h})(\tilde X^{l-1}R_K^{l,h})^\T\right),\nonumber\\
  \tilde S^l&:=\sum_{h=1}^H\tilde A^{l,h}\tilde X^{l-1}R_V^{l,h},\nonumber\\
  \tilde X^l&:=\left[\sigma(\tilde S^lR_g^l)\odot(\tilde S^lR_u^l)\right]R_d^l.
  \label{eq:glu_universal}
\end{align}
Its output is $\tilde X^LU$. Only $E$ depends on the target transformer.

\paragraph{Backward widths and probes.}
For positive integers $H,L,\din,d,r$, define
\begin{equation}
  s_L:=\din,\qquad q_l:=rs_l,\qquad
  s_{l-1}:=2Hd+2Hq_l\quad(l=L,\ldots,1).
  \label{eq:glu_widths}
\end{equation}
With $a:=2Hr\geq2$, this gives
\begin{equation}
  s_0=a^L\din+2Hd\frac{a^L-1}{a-1}
  =O\!\left((2Hr)^L(\din+d)\right).
  \label{eq:glu_width_closed}
\end{equation}
For each layer, partition $[q_l]$ into the consecutive groups
\[
  I_{l,j}:=\{(j-1)s_l+1,\ldots,js_l\},\qquad j\in[r].
\]
In the following concatenations, braces indexed by $h$ mean horizontal concatenation in increasing head order, preserving the displayed order within each head.

\begin{lemma}[Backward linear-probe simulation]
\label{lem:glu_probes}
For fixed universal parameters, define $P_l\in\R^{m\times s_l}$ by $P_L=U$ and
\begin{equation}
  P_{l-1}:=\left[
    \{R_Q^{l,h},R_K^{l,h}\}_{h\in[H]},\;
    \{R_V^{l,h}R_g^l,R_V^{l,h}R_u^l\}_{h\in[H]}
  \right].
  \label{eq:glu_probes}
\end{equation}
Suppose that $P_0$ has full column rank and that every matrix
\begin{equation}
  B_{l,j}:=R_d^l[I_{l,j},:]P_l\in\R^{s_l\times s_l}
  \label{eq:glu_local_matrix}
\end{equation}
is invertible. Then, for every target transformer $T\in\tf^{\sigma}_{H,L,\din,d,r}$, there exists $E\in\R^{\din\times m}$ such that $\tilde X^LU=T(X^0)$ for every context length $n\in\mathbb{N}$ and every input $X^0\in\R^{n\times\din}$.
\end{lemma}

\begin{proof}
Construct target probes $C_l\in\R^{\din\times s_l}$ backward, starting with $C_L=I_{\din}$. Given $C_l$, define for each $j\in[r]$
\begin{equation}
  \gamma_{l,j}:=W_d^l[j,:]C_l,\qquad
  \alpha_{l,j}:=\gamma_{l,j}B_{l,j}^{-1}
  \quad\text{in }\R^{1\times s_l}.
  \label{eq:glu_coefficients}
\end{equation}
For $i=(j-1)s_l+b\in I_{l,j}$, set
\begin{equation}
  \widehat W_g^l[:,i]:=W_g^l[:,j],\qquad
  \widehat W_u^l[:,i]:=\alpha_{l,j,b}W_u^l[:,j].
  \label{eq:glu_expanded_weights}
\end{equation}
The resulting matrices $\widehat W_g^l,\widehat W_u^l$ have dimensions $\din\times q_l$. Define the preceding target probe in the same column order as \eqref{eq:glu_probes}:
\begin{equation}
  C_{l-1}:=\left[
    \{W_Q^{l,h},W_K^{l,h}\}_{h\in[H]},\;
    \{M^{l,h}\widehat W_g^l,M^{l,h}\widehat W_u^l\}_{h\in[H]}
  \right].
  \label{eq:glu_target_probes}
\end{equation}
Because $P_0$ has full column rank, choose a left inverse $P_0^\dagger$ and set $E=C_0P_0^\dagger$. Thus $EP_0=C_0$.

We prove by induction that
\begin{equation}
  \tilde X^lP_l=X^lC_l\qquad(l=0,\ldots,L).
  \label{eq:glu_induction}
\end{equation}
The base case follows from $\tilde X^0=X^0E$. Assume the identity at layer $l-1$. The query and key columns in the probes give
\[
  \tilde X^{l-1}R_Q^{l,h}=X^{l-1}W_Q^{l,h},\qquad
  \tilde X^{l-1}R_K^{l,h}=X^{l-1}W_K^{l,h},
\]
so $\tilde A^{l,h}=A^{l,h}$ for every head. The remaining columns give, for $i=(j-1)s_l+b$,
\begin{align*}
  \tilde S^lR_g^l[:,i]
  &=\sum_{h=1}^H A^{l,h}X^{l-1}M^{l,h}W_g^l[:,j]
    =S^lW_g^l[:,j],\\
  \tilde S^lR_u^l[:,i]
  &=\alpha_{l,j,b}\sum_{h=1}^H A^{l,h}X^{l-1}M^{l,h}W_u^l[:,j]
    =\alpha_{l,j,b}S^lW_u^l[:,j].
\end{align*}
Let $z_{l,j}:=\sigma(S^lW_g^l[:,j])\odot(S^lW_u^l[:,j])\in\R^{n\times1}$. The universal hidden unit $i$ equals $\alpha_{l,j,b}z_{l,j}$, and consequently
\begin{align*}
  \tilde X^lP_l
  &=\sum_{j=1}^r z_{l,j}\alpha_{l,j}R_d^l[I_{l,j},:]P_l\\
  &=\sum_{j=1}^r z_{l,j}\alpha_{l,j}B_{l,j}
    =\sum_{j=1}^r z_{l,j}\gamma_{l,j}\\
  &=\left[\sigma(S^lW_g^l)\odot(S^lW_u^l)\right]W_d^lC_l
    =X^lC_l.
\end{align*}
This completes the induction. Since $P_L=U$ and $C_L=I_{\din}$, its final identity is $\tilde X^LU=X^L$. The universal probes depend only on the fixed universal parameters. The target probes and the embedding may also depend on the target parameters, but none of these objects depends on $n$ or $X^0$.
\end{proof}

\begin{theorem}[Fixed sparse parameters with SwiGLU blocks]
\label{thm:glu_sparse}
Fix positive integers $H,L,\din,d,r$ and an entrywise activation $\sigma$.
Let $s_l,q_l$ be defined by \eqref{eq:glu_widths}, and assume $m\geq s_0$.
There exist fixed parameters
\begin{align*}
R_Q^{l,h},R_K^{l,h}&\in\{0,1\}^{m\times d},
& R_V^{l,h}&\in\{0,1\}^{m\times m},\\
R_g^l,R_u^l&\in\{0,1\}^{m\times q_l},
& R_d^l&\in\{0,1\}^{q_l\times m},
\end{align*}
for $l\in[L]$, $h\in[H]$, and an unembedding matrix
$U\in\{0,1\}^{m\times\din}$, each with at most $m$ nonzero entries,
such that for every target transformer $T\in\tf^{\sigma}_{H,L,\din,d,r}$
there exists $E\in\R^{\din\times m}$ for which the universal computation
in \eqref{eq:glu_universal} satisfies
\[
\tilde X^LU=T(X^0)
\qquad\text{for all }n\in\mathbb{N},\ X^0\in\R^{n\times\din}.
\]
\end{theorem}

\begin{proof}
Let $e_1,\ldots,e_m$ be the standard basis of $\R^m$, and write $J_t:=[e_1,\ldots,e_t]$. We construct parameters for which $P_l=J_{s_l}$ at every layer. Set $U=J_{\din}$. At layer $l$, abbreviate $q=q_l$ and define, for $k\in[d]$,
\[
  R_Q^{l,h}[:,k]=e_{2(h-1)d+k},\qquad
  R_K^{l,h}[:,k]=e_{(2h-1)d+k}.
\]
Choose the two gate-input selectors
\[
  R_g^l=[e_1,\ldots,e_q],\qquad
  R_u^l=[e_{q+1},\ldots,e_{2q}].
\]
For $h\in[H]$ and $i\in[q]$, put
\[
  b_g(h,i):=2Hd+2(h-1)q+i,\qquad
  b_u(h,i):=2Hd+(2h-1)q+i,
\]
and define the value-routing matrices by
\begin{equation}
  R_V^{l,h}:=\sum_{i=1}^q
    \left(e_{b_g(h,i)}e_i^\T+e_{b_u(h,i)}e_{q+i}^\T\right).
  \label{eq:glu_sparse_routing}
\end{equation}
These assignments make the concatenation in \eqref{eq:glu_probes} exactly $J_{s_{l-1}}$. Finally, for $j\in[r]$ and $b\in[s_l]$, set
\[
  R_d^l[(j-1)s_l+b,:]=e_b^\T.
\]
Then $R_d^l[I_{l,j},:]P_l=I_{s_l}$ for every $l,j$, while $P_0=J_{s_0}$ has full column rank.

All coordinates used are at most $s_{l-1}\leq s_0\leq m$. Each query or key matrix has $d$ nonzeros; each gate, up, or down matrix has $q_l$; each value matrix has $2q_l$; and $U$ has $\din$. These counts are at most $m$, since $s_{l-1}=2Hd+2Hq_l\geq2q_l$. All entries are binary, and \Cref{lem:glu_probes} proves the claim.
\end{proof}

\begin{theorem}[Random parameters with SwiGLU blocks]
\label{thm:glu_random}
Fix positive integers $H,L,\din,d,r$ and an entrywise activation $\sigma$.
Let $s_l,q_l$ be defined by \eqref{eq:glu_widths}, and assume $m\geq s_0$.
Draw the parameters
\begin{align*}
R_Q^{l,h},R_K^{l,h}&\in\R^{m\times d},
& R_V^{l,h}&\in\R^{m\times m},\\
R_g^l,R_u^l&\in\R^{m\times q_l},
& R_d^l&\in\R^{q_l\times m},
\end{align*}
for $l\in[L]$, $h\in[H]$, and $U\in\R^{m\times\din}$ jointly from an
absolutely continuous distribution on the product of their matrix spaces.
Then, with probability $1$, for every target transformer
$T\in\tf^{\sigma}_{H,L,\din,d,r}$ there exists
$E\in\R^{\din\times m}$ for which the universal computation in
\eqref{eq:glu_universal} satisfies
\[
\tilde X^LU=T(X^0)
\qquad\text{for all }n\in\mathbb{N},\ X^0\in\R^{n\times\din}.
\]
\end{theorem}

\begin{proof}
Construct the probes by \eqref{eq:glu_probes}, starting from $P_L=U$. Fix in advance the square minor formed by the first $s_0$ rows of $P_0$, and denote it by $\overline P_0:=P_0[[s_0],:]$. As a function of all universal parameters, consider
\begin{equation}
  \Delta:=\det(\overline P_0)
    \prod_{l=1}^L\prod_{j=1}^r
      \det\!\left(R_d^l[I_{l,j},:]P_l\right).
  \label{eq:glu_determinant}
\end{equation}
Every probe entry is a polynomial in the universal parameters, so $\Delta$ is a polynomial. At the explicit binary assignment in the proof of \Cref{thm:glu_sparse}, $\overline P_0=I_{s_0}$ and every local matrix is $I_{s_l}$. Hence $\Delta=1$ at that assignment, proving that the polynomial is not identically zero.

The zero set of a nonzero real polynomial has Lebesgue measure zero. Joint absolute continuity therefore gives $\Delta\neq0$ with probability one. On this event, $P_0$ has full column rank and every local matrix in \eqref{eq:glu_local_matrix} is invertible. This event depends only on the fixed universal parameters. Thus \Cref{lem:glu_probes} applies simultaneously to every target transformer, giving the stated uniform almost-sure conclusion.
\end{proof}

\paragraph{Residual blocks.}
The same approach also accommodates a residual addition around each attention and feed-forward sublayer. Retain the attention definitions above, and replace the target updates by
\begin{align}
  Y^l&:=X^{l-1}+\sum_{h=1}^H A^{l,h}X^{l-1}M^{l,h},\nonumber\\
  X^l&:=Y^l+\left[\sigma(Y^lW_g^l)\odot(Y^lW_u^l)\right]W_d^l.
  \label{eq:glu_residual_target}
\end{align}
The universal model uses the corresponding updates
\begin{align}
  \tilde Y^l&:=\tilde X^{l-1}
       +\sum_{h=1}^H\tilde A^{l,h}\tilde X^{l-1}R_V^{l,h},\nonumber\\
  \tilde X^l&:=\tilde Y^l
       +\left[\sigma(\tilde Y^lR_g^l)\odot(\tilde Y^lR_u^l)\right]R_d^l.
  \label{eq:glu_residual_universal}
\end{align}
Denote the target class in \eqref{eq:glu_residual_target} by
$\tf^{\sigma,\mathrm{res}}_{H,L,\din,d,r}$.

\begin{theorem}[Residual SwiGLU blocks]
\label{thm:glu_residual}
Fix positive integers $H,L,\din,d,r$ and an entrywise activation $\sigma$.
Define the backward widths by
\begin{equation}
  s_L:=\din,\qquad q_l:=rs_l,\qquad
  s_{l-1}:=2Hd+(H+1)s_l+2(H+1)q_l,
  \label{eq:glu_residual_widths}
\end{equation}
for $l=L,\ldots,1$, and assume $m\geq s_0$.
There exist fixed parameters
\[
\{R_Q^{l,h},R_K^{l,h},R_V^{l,h}\}_{l\in[L],h\in[H]},
\qquad\{R_g^l,R_u^l,R_d^l\}_{l\in[L]},\qquad U,
\]
with the matrix dimensions in \Cref{thm:glu_sparse} using these widths
$q_l$, entries in $\{0,1\}$, and at most $m$ nonzero entries in each
matrix, such that for every target transformer
$T\in\tf^{\sigma,\mathrm{res}}_{H,L,\din,d,r}$ there exists
$E\in\R^{\din\times m}$ for which the residual universal computation
in \eqref{eq:glu_residual_universal} satisfies
\[
\tilde X^LU=T(X^0)
\qquad\text{for all }n\in\mathbb{N},\ X^0\in\R^{n\times\din}.
\]
Moreover, if all these parameters are drawn jointly from an absolutely
continuous distribution on the product of their real matrix spaces, then,
with probability $1$, the same simulation property holds: for every such
$T$ there exists an embedding $E$ that gives the displayed equality for
all context lengths and inputs.
\end{theorem}

\begin{proof}
We first establish the sufficient probe conditions, and then construct a simultaneous witness for them. Set $P_L=U$ and form each $P_{l-1}$ by horizontally concatenating the following three blocks in order:
\begin{align}
  &\left[\{R_Q^{l,h},R_K^{l,h}\}_{h\in[H]}\right],\nonumber\\
  &\left[P_l,\;\{R_V^{l,h}P_l\}_{h\in[H]}\right],\nonumber\\
  &\left[R_g^l,R_u^l,\;
      \{R_V^{l,h}R_g^l,R_V^{l,h}R_u^l\}_{h\in[H]}\right].
  \label{eq:glu_residual_probes}
\end{align}
The total number of columns is $s_{l-1}$ in \eqref{eq:glu_residual_widths}. Suppose $P_0$ has full column rank and all $B_{l,j}=R_d^l[I_{l,j},:]P_l$ are invertible. Starting with $C_L=I_{\din}$, define $\gamma_{l,j},\alpha_{l,j},\widehat W_g^l,\widehat W_u^l$ by \eqref{eq:glu_coefficients}--\eqref{eq:glu_expanded_weights}. Define $C_{l-1}$ by the corresponding concatenation
\begin{align}
  &\left[\{W_Q^{l,h},W_K^{l,h}\}_{h\in[H]}\right],\nonumber\\
  &\left[C_l,\;\{M^{l,h}C_l\}_{h\in[H]}\right],\nonumber\\
  &\left[\widehat W_g^l,\widehat W_u^l,\;
      \{M^{l,h}\widehat W_g^l,M^{l,h}\widehat W_u^l\}_{h\in[H]}\right].
  \label{eq:glu_residual_target_probes}
\end{align}
Choose $E=C_0P_0^\dagger$. The same induction as in \Cref{lem:glu_probes} first gives $\tilde A^{l,h}=A^{l,h}$. The second block of probes then gives
\[
  \tilde Y^lP_l
  =X^{l-1}C_l+\sum_{h=1}^H A^{l,h}X^{l-1}M^{l,h}C_l
  =Y^lC_l.
\]
For $i=(j-1)s_l+b$, the third block gives
\[
  \tilde Y^lR_g^l[:,i]=Y^lW_g^l[:,j],\qquad
  \tilde Y^lR_u^l[:,i]=\alpha_{l,j,b}Y^lW_u^l[:,j].
\]
Writing $z_{l,j}:=\sigma(Y^lW_g^l[:,j])\odot(Y^lW_u^l[:,j])$, we obtain
\begin{align*}
  \tilde X^lP_l
  &=Y^lC_l+\sum_{j=1}^r z_{l,j}\alpha_{l,j}B_{l,j}\\
  &=Y^lC_l+\sum_{j=1}^r z_{l,j}W_d^l[j,:]C_l
    =X^lC_l.
\end{align*}
This proves the residual version of the probe lemma, including uniformity over context length and input.

We now give a deterministic binary witness. For any ordered set $\mathcal S$ of coordinates, write $J_{\mathcal S}$ for the matrix whose columns are the corresponding standard basis vectors in that order. Begin with $P_L=U=J_{[\din]}$, and proceed backward. Suppose $P_l=J_{\mathcal S}$ for an ordered set $\mathcal S\subseteq[s_0]$ of size $s_l$. Choose ordered coordinate sets
\[
  \mathcal G,\mathcal U,\quad
  \{\mathcal Q_h,\mathcal K_h\}_{h\in[H]},\quad
  \{\mathcal S_h,\mathcal G_h,\mathcal U_h\}_{h\in[H]}
\]
of respective sizes $q_l,q_l,d,d,s_l,q_l,q_l$, all disjoint from one another and from $\mathcal S$. To make the assignment deterministic, allocate these sets in the displayed order using the smallest unused coordinates of $[s_0]$, with increasing order within each set. This is possible because the total number of coordinates, including $\mathcal S$, is exactly $s_{l-1}\leq s_0$.

Set $R_g^l=J_{\mathcal G}$, $R_u^l=J_{\mathcal U}$, $R_Q^{l,h}=J_{\mathcal Q_h}$, and $R_K^{l,h}=J_{\mathcal K_h}$. For each head, define
\begin{equation}
  R_V^{l,h}:=
    J_{\mathcal S_h}P_l^\T
    +J_{\mathcal G_h}J_{\mathcal G}^\T
    +J_{\mathcal U_h}J_{\mathcal U}^\T.
  \label{eq:glu_residual_witness}
\end{equation}
The disjoint domains imply
\[
  R_V^{l,h}P_l=J_{\mathcal S_h},\qquad
  R_V^{l,h}R_g^l=J_{\mathcal G_h},\qquad
  R_V^{l,h}R_u^l=J_{\mathcal U_h}.
\]
Consequently, \eqref{eq:glu_residual_probes} consists of $s_{l-1}$ distinct coordinate columns and defines the next selector $P_{l-1}$. Set $R_d^l[I_{l,j},:]=P_l^\T$ for every $j$. Then every local matrix $B_{l,j}$ is $I_{s_l}$. The construction is consistent across layers because their parameters are untied.

All parameters are binary. Each value matrix has $s_l+2q_l$ nonzero entries, each gate, up, or down matrix has $q_l$, each query or key matrix has $d$, and $U$ has $\din$. The width recurrence bounds all these counts by $m$. At the end, $P_0$ selects $s_0$ distinct coordinates from $[s_0]$, so its first $s_0$ rows form a permutation matrix. This proves the sparse assertion via the residual probe conditions.

Finally, use the determinant polynomial \eqref{eq:glu_determinant} with the residual probes and widths, again fixing the first $s_0$ rows for $\overline P_0$. The recursion \eqref{eq:glu_residual_probes} makes every probe entry polynomial in the universal parameters. At the deterministic assignment above, $\det(\overline P_0)=\pm1$ and every local determinant is $1$. Hence the polynomial is nonzero. The measure-zero argument from \Cref{thm:glu_random} now proves the almost-sure assertion, on a single event that is independent of the target transformer.
\end{proof}

\paragraph{Scope.}
These results concern bias-free blocks without normalization and with untied universal parameters across layers. The random theorems use the displayed layer-dependent hidden widths $q_l=rs_l$. The sparse constructions may be padded to a common larger feed-forward width by adding units with zero gate, up, and down parameters.

\section{Lower bound on embedding dimension}
\label{app:lower_bound}

We prove a lower bound for our approach even when all the attention matrices are matched. 

\begin{assumption}
    For each attention head in the random/designed transformer, the attention matrix is equal to that of the target transformer, i.e. $\tilde A^{l,h}=A^{l,h}$, and $\forall l\in [L], h\in [H], A^{l,h}$ is arbitrary row-stochastic matrix, i.e. each row of $A^{l,h}$ sum up to one. 
\end{assumption}

\begin{theorem}[Lower bound assuming matching attention]
\label{thm:lower_bound_matching_attention}
Let $\din=1$. Suppose fixed matrices $R_V^{l,h}\in\R^{m\times m}$,
$l\in[L]$, $h\in[H]$, have the following property for every context length
$n\in\mathbb{N}$. For any target transformer with attention matrices
$\{A^{l,h}\}_{l,h}$ and value and output matrices
$\{W_V^{l,h},W_O^{l,h}\}_{l,h}$, there exist
$E\in\R^{\din\times m}$ and $U\in\R^{m\times\din}$ such that
$Z=\tilde X^LU$ equals $Y=X^L$ for every input
$X^0\in\R^{n\times\din}$. Then $m\ge H^{L/2}$.
\end{theorem}

\begin{proof}
    We prove by contradiction. Recall that 
    \begin{align*}
    X^L & = \sum_{(h_1,\dotsc,h_L) \in \sbr{H}^L} A^{L,h_L} \dotsm A^{1,h_1} X^0 W_V^{1,h_1} W_O^{1,h_1} \dotsm W_V^{L,h_L} W_O^{L,h_L} \\
    \tilde X^L & = \sum_{(h_1,\dotsc,h_L) \in \sbr{H}^L} \tilde A^{L,h_L} \dotsm \tilde A^{1,h_1} X^0 E R_V^{1,h_1} \dotsm R_V^{L,h_L} 
  \end{align*}
  Assuming the attention matrices are the same, for $\tilde X^LU=X^L$, 
  \begin{equation} \label{eq:lower_bound_matching_term}
      \sum_{(h_1,\dotsc,h_L) \in \sbr{H}^L} A^{L,h_L} \dotsm A^{1,h_1} X^0 \left(\prod_{l=1}^L W_V^{l, h_l}W_O^{l, h_l} - E \prod_{l=1}^L R_V^{l, h_l}U\right)=0
  \end{equation}
  Let $\textbf{h}:=(h_1,\ldots,h_L)$ and
  \begin{align*}
    \alpha_{\textbf{h}}
    &:=\prod_{l=1}^L W_V^{l,h_l}W_O^{l,h_l}
       -E\prod_{l=1}^L R_V^{l,h_l}U,\\
    A_{\textbf{h}}
    &:=A^{L,h_L}\dotsm A^{1,h_1}.
  \end{align*}
  The embedding and unembedding matrices may depend on the target attention matrices. We therefore use a single collection of attention matrices for which output equality forces all coefficients $\alpha_{\textbf{h}}$ to vanish simultaneously.

  \begin{lemma}
    Let $N:=H^L$ and $n:=2N$. There exist $n\times n$ row-stochastic matrices $\{A^{l,h}\}_{l\in[L],h\in[H]}$ and an input $X^0\in\R^{n\times 1}$ such that the $N$ vectors
    \[
      \left\{A^{L,h_L}\dotsm A^{1,h_1}X^0:
      (h_1,\dots,h_L)\in[H]^L\right\}
    \]
    are linearly independent.
  \end{lemma}

  \begin{proof}
    Index $N$ blocks of size $2$ by $s=(s_1,\dots,s_L)\in[H]^L$. Let
    \[
      I=\begin{pmatrix}1&0\\0&1\end{pmatrix},
      \qquad
      B=\begin{pmatrix}0&1\\0&1\end{pmatrix},
      \qquad
      e_1=\begin{pmatrix}1\\0\end{pmatrix}.
    \]
    For each $l\in[L]$ and $h\in[H]$, define the block-diagonal matrix
    \[
      A^{l,h}:=\bigoplus_{s\in[H]^L} A_s^{l,h},
      \qquad
      A_s^{l,h}:=
      \begin{cases}
        I,&h=s_l,\\
        B,&h\neq s_l.
      \end{cases}
    \]
    These matrices are row-stochastic. Take $X^0$ to be the column vector obtained by stacking $N$ copies of $e_1$, in the same block order.

    Fix a path $\textbf{h}=(h_1,\dots,h_L)$. In block $s=\textbf{h}$, every factor in $A^{L,h_L}\dotsm A^{1,h_1}$ equals $I$. In every other block, at least one factor equals $B$. Since
    \[
      IB=BI=BB=B,\qquad Ie_1=e_1,\qquad Be_1=0,
    \]
    the block indexed by $s$ in $A_{\textbf{h}}X^0$ is
    \[
      (A_{\textbf{h}}X^0)_s=
      \begin{cases}
        e_1,&s=\textbf{h},\\
        0,&s\neq\textbf{h}.
      \end{cases}
    \]
    Thus the $N$ path vectors are nonzero and have disjoint supports, so they are linearly independent.
  \end{proof}

Fix the single collection of attention matrices and input from this lemma. For any target value weights, choose $E,U$ as in the theorem for this collection. In the block indexed by $\textbf{h}$, \Cref{eq:lower_bound_matching_term} becomes $\alpha_{\textbf{h}}e_1=0$. Consequently, $\alpha_{\textbf{h}}=0$ for every $\textbf{h}\in[H]^L$, without changing the attention matrices or the chosen $E,U$. Next, we show that if $m<H^{L/2}$, there are target value weights for which this is impossible.

\begin{lemma}
  If $m < H^{L/2}$, then for each collection of $m \times m$ matrices $\set{R^{l,h}}_{(l,h) \in [L] \times [H]}$, there exists a collection of $\din \times \din$ matrices $\set{W^{l,h}}_{(l,h) \in [L] \times [H]}$ such that, for each pair of $\din \times m$ matrices $E, U$, there exists $(h_1,\dotsc,h_L) \in [H]^L$ such that
  \begin{equation*}
    E \prod_{l=1}^L R^{l,h_l} U^\T \neq \prod_{l=1}^L W^{l,h_l}
    .
  \end{equation*}
\end{lemma}

\begin{proof}
  Fix a collection of $m \times m$ matrices $\set{R^{l,h}}_{(l,h) \in [L] \times [H]}$.
  Consider the system of $H^L$ linear equations $Ax = b$ in $m^2$ unknowns $x$, where the equation corresponding to $(h_1,\dotsc,h_L) \in [H]^L$ is
  \begin{equation*}
    \tr\del*{\prod_{l=1}^L R^{l,h_l} \mat(x)} = \prod_{l=1}^L W^{l,h_l}
    .
  \end{equation*}
  Here, $\mat(x)$ reshapes $x$ to a $m \times m$ matrix.
  The matrix $A$ is determined by $\set{R^{l,h}}_{(l,h) \in [L] \times [H]}$; below, we show how to choose $\set{W^{l,h}}_{(l,h) \in [L] \times [H]}$, which in turn determines the right-hand side vector $b$.

  The rank of $A$ is at most the number of columns, which is $m^2$; that number is at most $H^L - 1$ by the assumption $m < H^{L/2}$.
  So by the rank-nullity theorem, the left nullspace of $A$ has dimension at least $1$.
  Let $y$ be a nonzero vector in the left nullspace of $A$, and let $y_{h_1^*,\dotsc,h_L^*}$ be a nonzero component of $y$.
  For each $(l,h) \in [L] \times [H]$, choose
  \begin{equation*}
    W^{l,h} :=
    \begin{cases}
      1 & \text{if $h = h_l^*$} ; \\
      0 & \text{otherwise} ;
    \end{cases}
  \end{equation*}
  and let $b$ be the corresponding right-hand side vector.
  We have
  \begin{equation*}
    \ip{y,b}
    = \sum_{(h_1,\dotsc,h_L) \in [H]^L} y_{h_1,\dotsc,h_L} \times \prod_{l=1}^L W^{l,h_l}
    = y_{h_1^*,\dotsc,h_L^*} \times 1
    \neq 0 .
  \end{equation*}
  Therefore, $b$ has a nonzero component in this left nullspace, which implies that $b$ is not in the column space of $A$, i.e., $Ax = b$ has no solution.
  Since $U^\T E$ is a $m \times m$ matrix, it follows that there exists some $(h_1,\dotsc,h_L) \in [H]^L$ such that
  \begin{equation*}
    E \prod_{l=1}^L R^{l,h_l} U^\T = 
    \tr\del*{\prod_{l=1}^L R^{l,h_l} U^\T E} \neq \prod_{l=1}^L W^{l,h_l} 
  \end{equation*}
\end{proof}
The proof of \Cref{thm:lower_bound_matching_attention} follows from this. 
\end{proof}

The lower bound also works for the weight-tying (looped) transformer, by choosing 
$$ W^{h} :=
\begin{cases}
  1 & \text{if $h = h_l^*$} ; \\
  0 & \text{otherwise} .
\end{cases}
$$

\section{Additional ablation studies}

Besides parenthesis balancing and $2$-hop induction heads, we also perform ablation studies on
$1$-hop, $3$-hop, and $4$-hop induction heads tasks. These experiments follow the
same protocol as in the parenthesis balancing and $2$-hop induction heads: for the sparse and random universal transformers,
only the embedding and unembedding matrices are trained, while all internal attention parameters
are kept fixed. When MLP layers are added, they are randomly initialized and also kept fixed. The results are shown in \Cref{tab:induction_head_1hop,tab:induction_head_3_4_hop}. These results are consistent with \Cref{tab:induction_head_2hop}, that while it's hard for the attention-only universal transformers to find the optimal solution, adding residual connections and layer normalization can make the optimization much easier. 

\begin{table}[h]
\centering
\caption{Induction head (1-hop) results for $H=2$, $L=2$, head\_dim $=28$, embedding\_dim $=456$.}
\label{tab:induction_head_1hop}
\begin{tabular}{lccc}
\toprule
Model & Sparse & Random & Fully-trained \\
\midrule
Attention-only & 30.2\% & 29.4\% & 8.4\% \\
+ residual & 42.1\% & 5.5\% & 67.5\% \\
+ residual and LN & 100\% & 96.1\% & 100\% \\
+ MLP & 61.0\% & 51.3\% & 100\% \\
+ MLP and LN & 100\% & 96.0\% & 100\% \\
\bottomrule
\end{tabular}
\end{table}

\begin{table}[h]
\centering
\caption{3-hop and 4-hop results for $H=2$, $L=4$, head\_dim $=30$, embedding\_dim $=2280$.}
\label{tab:induction_head_3_4_hop}

\textbf{3-hop}

\begin{tabular}{lccc}
\toprule
Model & Sparse & Random & Fully-trained \\
\midrule
Attention-only & 39.2\% & 32.6\% & 99.5\% \\
+ residual & 53.6\% & 43.9\% & 99.9\% \\
+ residual and LN & 100\% & 92.6\% & 100\% \\
+ MLP & 56.9\% & 57.5\% & 100\% \\
+ MLP and LN & 100\% & 98.7\% & 100\% \\
\bottomrule
\end{tabular}

\vspace{1em}

\textbf{4-hop}

\begin{tabular}{lccc}
\toprule
Model & Sparse & Random & Fully-trained \\
\midrule
Attention-only & 39.4\% & 35.2\% & 95.6\% \\
+ residual & 43.9\% & 41.6\% & 97.9\% \\
+ residual and LN & 100\% & 89.2\% & 100\% \\
+ MLP & 44.7\% & 48.8\% & 97.4\% \\
+ MLP and LN & 82.8\% & 70.9\% & 100\% \\
\bottomrule
\end{tabular}

\end{table}

\section{Experimental details}

\subsection{Data generation}

\paragraph{Parenthesis balancing.}
The data is generated following the same way as \citep{zhong2024algorithmic}.
We generate sequences over a four-symbol vocabulary consisting of padding/EOF, left parenthesis, right parenthesis, and a query token, encoded as $0,1,2,3$, respectively. Each example is constructed by first sampling a parenthesis string of length at most $2\,\texttt{max\_len}$. With probability $1/3$, we sample a completely random parenthesis string of uniformly random length; otherwise, we recursively generate a balanced parenthesis string with a uniformly sampled number of matching pairs. To obtain both positive and negative instances, the sampled string is optionally corrupted by local mutations. These mutations include swapping two positions in the string and flipping parentheses from left to right or right to left, with a geometrically distributed number of repetitions. After the parenthesis string $s$ is generated, we compute whether it is balanced using the standard stack-depth criterion: the running depth must never become negative and must end at zero. We then append a query token ``?'' followed by the target answer token. The answer is ``)'' if $s$ is balanced and ``('' otherwise. Finally, the sequence is padded with zeros to a fixed length $2\,\texttt{max\_len}+3$. The model is trained in a next-token prediction format: for each padded sequence, the input is the sequence excluding the final token, and the target is the same sequence shifted one position to the left.

\paragraph{$k$-hop induction heads.} For this task, we generate the data following \citep{sanford2024}. We generate sequence-modeling examples over a vocabulary consisting of character tokens, induction-query tokens, a blank token, and a special ``does-not-exist'' token. The character vocabulary contains \texttt{char\_tokens} symbols, denoted by consecutive lowercase letters. For each example, we first sample a random character string $x_1,\ldots,x_{\texttt{seq\_len}}$, where the first character is sampled uniformly and each subsequent character is sampled uniformly subject to being different from the previous character. This prevents trivial adjacent repetitions. We then construct a sequence of induction-hop targets. The $0$-hop string is the original string itself. Given the current hop string and its associated source indices, the next-hop string is obtained positionwise as follows: for position $i$, we look backward in the original random string for the most recent previous occurrence of the current symbol before its current source index. If such an occurrence exists, the next-hop target is the character immediately following that occurrence; otherwise, the target is the special ``does-not-exist'' symbol. Repeating this procedure produces targets for all hop depths up to \texttt{max\_hops}. For each training example, the queried hop depth is either fixed or sampled uniformly from $\{\texttt{min\_hops},\ldots,\texttt{max\_hops}\}$ and encoded as a leading induction-query token. The model input consists of this query token followed by the random character string with the last two characters removed. The output is a blank token followed by the corresponding hop-target string, also truncated by two characters. Thus the task is trained in a full-sequence prediction format: after reading the hop query and the character context, the model must output, at every position, the result of applying the specified number of induction steps to the underlying sequence.

\subsection{Training details}

All the experiments are conducted on one single NVIDIA RTX PRO 6000 Blackwell 96GB. This project roughly takes 14 GPU days. 

The embedding and unembedding matrices are initialized from standard normal distribution with standard deviation $0.02$. We use AdamW as optimizer and LambdaLR as learning rate scheduler. We adopt online training, i.e. at each step we generate fresh data. 

The hyperparameters for parenthesis balancing and $k$-hop induction heads task can be found in \Cref{tab:hyperparameters_parenthesis_balancing,tab:hyperparameters_k-hop}. For $k=1, 2, 3, 4$, the training steps are $10000, 15000, 60000, 60000$ respectively.

\begin{table}[h]
\centering
\begin{tabular}{ll}
\toprule
\textbf{Hyperparameter} & \textbf{Value} \\
\midrule
Embedding dimension $m$ & $1024$ \\
Head dimension $d$ & $24$ \\
Depth $L$ & $2$ \\
Number of heads $H$ & $4$ \\
Vocabulary size & $4$ \\
Activation function of MLP & GeLU \\
Layer norm $\epsilon$ & $10^{-5}$ \\
Learning rate & $10^{-3}$ \\
Warmup steps & $50$ \\
Training steps & $10^4$ \\
Batch size & $1000$ \\
\bottomrule
\end{tabular}
\caption{Model and training hyperparameters for parenthesis balancing.}
\label{tab:hyperparameters_parenthesis_balancing}
\end{table}

\begin{table}[h]
\centering
\begin{tabular}{ll}
\toprule
\textbf{Hyperparameter} & \textbf{Value} \\
\midrule
Number of heads $H$ & $2$ \\
Number of character tokens & $4$ \\
Vocabulary size & $30$ \\
Activation function of MLP & GeLU \\
Layer norm $\epsilon$ & $10^{-5}$ \\
Learning rate & $10^{-4}$ \\
Warmup steps & $1000$ \\
Batch size & $128$ \\
\bottomrule
\end{tabular}
\caption{Model and training hyperparameters for $k$-hop induction heads.}
\label{tab:hyperparameters_k-hop}
\end{table}